\documentclass[conference]{IEEEtran}
\IEEEoverridecommandlockouts
\usepackage{cite}
\usepackage{amsmath,amssymb,amsfonts}
\usepackage{algorithmic}
\usepackage{textcomp}
\usepackage{xcolor}
\usepackage{placeins}
\usepackage{soul}
\def\BibTeX{{\rm B\kern-.05em{\sc i\kern-.025em b}\kern-.08em
    T\kern-.1667em\lower.7ex\hbox{E}\kern-.125emX}}

\usepackage{graphicx}
\usepackage[caption=false,font=footnotesize]{subfig}
\usepackage{booktabs}
\usepackage{adjustbox}
\usepackage[a4paper,margin=1in]{geometry}
\usepackage[T1]{fontenc}
\usepackage[utf8]{inputenc}
\usepackage{lmodern}
\usepackage{microtype}
\usepackage{amsmath,amssymb,amsthm,mathtools}
\usepackage{enumitem}
\usepackage{xcolor}
\usepackage{hyperref}
\usepackage[nameinlink,noabbrev,capitalize]{cleveref}

\hypersetup{
  colorlinks=true,
  linkcolor=blue,
  citecolor=blue,
  urlcolor=blue
}

\newtheorem{theorem}{Theorem}
\newtheorem{lemma}{Lemma}
\newtheorem{corollary}{Corollary}
\theoremstyle{definition}
\newtheorem{definition}{Definition}
\newtheorem{remark}{Remark}
\newcommand{\R}{\mathbb{R}}
\newcommand{\Sph}{\mathbb{S}}

\newcommand{\angles}[2]{\angle\!\left(#1,#2\right)}

\newcommand{\covrad}{\operatorname{covrad}}

\newcommand{\Pcode}{\mathcal P}
\newcommand{\fp}{\mathrm{fp}}

\DeclareMathOperator*{\argmin}{argmin}

\begin{document}
\pagestyle{plain}
\title{Direction-Preserving Number Representations}
\author{
\IEEEauthorblockN{Bardia Zadeh\IEEEauthorrefmark{1},
George A. Constantinides\IEEEauthorrefmark{1}}

\IEEEauthorblockA{\IEEEauthorrefmark{1}
Department of Electrical and Electronic Engineering\\
Imperial College London\\
London, United Kingdom\\
Email: \{bm1220, g.constantinides\}@imperial.ac.uk \\
ORCID: 0009-0001-7766-7060, 0000-0002-0201-310X}
}

\maketitle
\begin{abstract}
Low-precision number formats are widely used in modern machine learning systems due to their efficiency. Accurate direction representation is key to the accuracy of vector operations. This work precisely explores the extent to which the direction of a vector can be represented by selecting its scalar elements from a common finite alphabet of a given size. This is standard practice in machine learning, where low-precision significands may be narrow-width floating-point or integer values. A geometric framework is introduced for analyzing the directional coverage of such product-structured codes. 

This work analytically quantifies the suboptimality gap between such product-structured codes and spherical codes for the vector as a whole, in both low and asymptotically high dimensions. Furthermore, within the product code class, it is proven that the standard formats of two's complement, fixed-point, and floating-point are suboptimal, again with quantified gap, pointing to the potential to develop new scalar number formats.

Such scalar alphabets are numerically optimized across multiple block dimensions for directional coverage, including the dimension used in NVIDIA’s NVFP4 format. Experimental results are presented comparing the performance of standard formats and the optimized alphabet. We find that for four bits, NVIDIA's choice of E2M1 closely approximates the optimized alphabet, providing a geometric explanation for its strong performance in low-precision machine learning workloads and an analytical understanding of the link between that superiority and block size.

We provide open-source formal proofs in Lean for the theorems in this work, along with the experimental code and the optimized alphabets obtained.
\end{abstract}

\section{Introduction}
\label{sec:intro}

\begin{figure}
    \centering
    \includegraphics[width=0.99\linewidth]{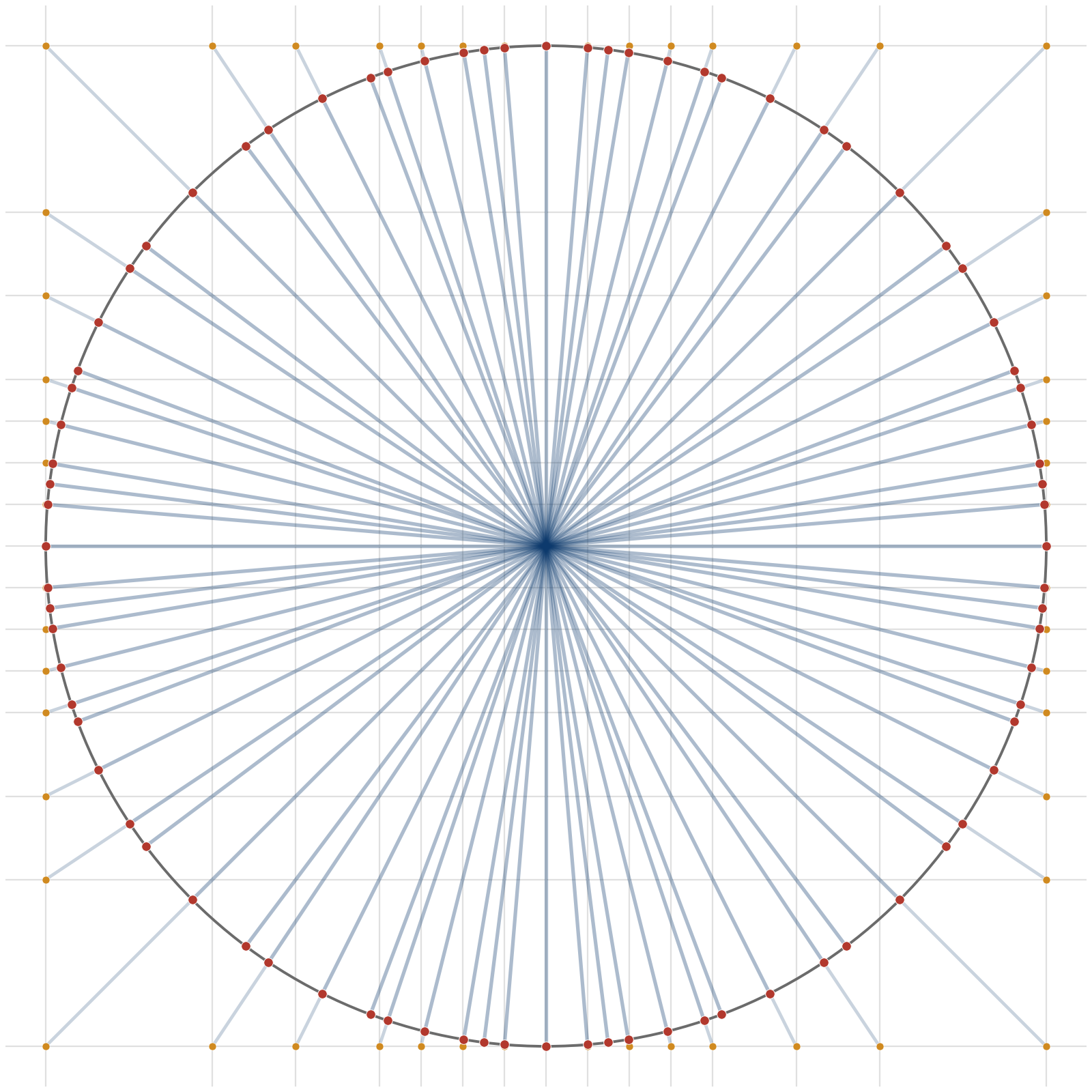}
    \caption{The directional coverage for a product code formed from two elements in the 4-bit E2M1 format. Each intersection of the grid lines corresponds to a representable direction through the product structure, projected onto the circle as a red dot. It can be seen that the red dots are not equally spaced in angle, leading to gaps in angular coverage.}
    \label{fig:fp4-e2m1-2d}
\end{figure}
The increasing computational demands of modern machine learning systems have driven widespread adoption of low-precision numerical representations during both inference and training. Low-precision representations enable significant savings in power consumption and resource utilization along with increases in throughput, often with minimal degradation in model task accuracy~\cite{quant_inf_survey},~\cite{nvfp4_training},~\cite{gholami2022survey}.

To enable low-precision representation, per-element scalar quantization is commonly used in which real-valued vectors are mapped elementwise to a scalar alphabet, resulting in a product code of quantized elements. 

When evaluating these quantized vectors, most work focuses on the mean-squared error or other norm-based metrics~\cite{quant_inf_survey}. However, the new industrial dominance of block-based number representations such as block minifloat~\cite{block-minifloat}, MX~\cite{ocp_mx}, and the draft IEEE P3109 standard~\cite{IEEE_SA_P3109_Working_Group_Interim_Report_on_2026}, where scaling of a (sub)vector is recorded separately from the vector itself, suggests an alternative metric: directional fidelity. 

The product structure of the quantized vectors fundamentally restricts the geometry of the representation space: the set of achievable directions is not arbitrary, but is instead determined by the scalar alphabet. Figure~\ref{fig:fp4-e2m1-2d} shows the coverage of a product code formed by two 4-bit E2M1\footnote{E2M1 refers to a floating-point format with exponent width 2 and mantissa width 1. This format is used in NVFP4~\cite{nvfp4_dtype}. E2M1 does not support infinities or NaN values, and is part of the OCP standard~\cite{ocp_mx}.} elements.

These observations motivate a shift in perspective for the design of number systems themselves: rather than focusing on traditional questions of dynamic range and precision, we ask whether \textit{scalar representations can be designed to preserve vector directions when those scalars populate the elements of a vector}. 
We offer the following contributions in this work:
\begin{enumerate}
    \item An analytical proof of the angular error gap, in both low and high dimensions, between such product-structured codes and optimal spherical codes.
    \item An analytical proof that among all product codes for a given bit-width, two's complement, fixed-point, and floating-point are suboptimal, with quantified suboptimality gap. 
    \item A numerical optimization approach to obtain a direction-preserving alphabet, and an empirical evaluation of contemporary 4-bit formats against this benchmark.
    \item Open-source access to the following: proofs of our theorems in Lean, experimental code for evaluating various formats, and the optimized number formats themselves\footnote{https://github.com/bardia01/Direction-Preserving-Number-Representations}.
\end{enumerate}

\section{Background and Related Work}

\subsection{Scalar Quantization of Vectors}
\label{sec:scalar_quant}
Scalar quantization applied to a vector independently maps each element to a discrete set of values. Formally, given a vector $x \in \mathbb{R}^d$, we apply a scalar quantizer $Q(\cdot)$ element-wise, producing a quantized vector $\hat{x}$ such that $\hat{x}_i = Q(x_i)$ for each dimension $i$. The quantizer $Q(\cdot)$ can be uniform, or non-uniform, symmetric or asymmetric~\cite{gholami2022survey}. Within the scope of this work, a scalar quantizer maps an element to its nearest neighbor in the alphabet: $Q : \mathbb{R} \to \mathcal{S}$, where $\mathcal{S} \subset \mathbb{R}$ denotes the set of representable values. We assume the existence of a tie-break rule, but the nature of that rule is irrelevant to this work. Thus we may write:
\[
Q(x) = \arg\min_{s \in \mathcal{S}} |x - s|.
\]
The resulting quantized vector is drawn from the Cartesian product $\mathcal{S}^n$; in this sense, it forms a product code.
The set of values a scalar quantizer can map to is referred to as the scalar alphabet. 

\subsection{Spherical codes}
\label{sec:spherical_code}
Spherical codes~\cite{wyner-sphere-packing} are sets of points distributed on the surface of a unit sphere. In this work, spherical codes serve as an unconstrained benchmark. For a fixed number of representable directions, a spherical code may place points arbitrarily on the sphere and therefore captures the best achievable directional coverage. In contrast, the product code direction sets considered in this paper are constrained in their angular coverage capabilities due to their construction by per-coordinate scalar choices drawn from a shared finite alphabet.

\subsection{Shape--gain quantization and block-formats}
The decomposition of a vector into its magnitude (gain) and direction (shape) is a classical approach in vector quantization~\cite{gray_vq}. Shape--gain quantizers encode the two components separately, typically using scalar quantization for the gain and spherical codes for the shape component. 

By factoring out a common scale across groups of elements, block-based number representations~\cite{block-minifloat},~\cite{ocp_mx},~\cite{IEEE_SA_P3109_Working_Group_Interim_Report_on_2026} may be viewed as instances of shape--gain quantization in which the shape component is constrained to a product-structured code.

In this paper we investigate the implications of using product structure to represent the shape component, as opposed to the standard practice of using spherical codes.

\subsection{Direction Preservation in Low-Precision Representations}

Recent work has increasingly highlighted the importance of preserving directional information in low-precision machine learning workloads. Several works explicitly incorporate direction or inner-product preservation into the quantization objective. TurboQuant~\cite{turboquant} includes a secondary optimization stage that directly minimizes inner-product error between original and quantized vectors. PolarQuant~\cite{polarquant} parametrizes vectors in polar coordinates and quantizes magnitude and direction separately to better preserve angular structure, similarly to traditional shape-gain quantizers.

These approaches focus on designing quantization schemes that preserve directional properties. Instead, we analytically characterize and study the representational limits imposed by a fixed class of quantizers, namely those derived from a product structure of scalar alphabets. Rather than proposing a new quantization algorithm, we analyze how well such representations can approximate directions in the worst case, thereby providing a complementary, geometry-driven perspective.

\section{Definitions}
The theoretical analysis begins by introducing key definitions and some general remarks. This section aims to introduce the geometric framework for analyzing product code directional coverage.

\begin{definition}[Finite scalar alphabets]
For $q\ge 2$, let
\[
\mathcal A_q:=\{A\subset \R: A \text{ is finite and } |A|=q\}.
\]
\end{definition}

\begin{definition}[Product code direction set]
\label{def:prod_dir}
Let $n\ge 2$ and let $A\subset \R$ be finite. Define
\[
\Pcode_n(A)
:=
\left\{
\frac{x}{\|x\|_2}:x\in A^n\setminus\{0\}
\right\}
\subset \Sph^{n-1},
\]
where $\Sph^{n-1}$ denotes the $(n-1)$-dimensional sphere, embedded in $n$-dimensional ambient space.
\end{definition}

\begin{definition}[Product code covering objective]
\label{def:prod-code-objective}
For $n\ge 2$ and finite $A\subset \R$, define
\[
F_n(A)
:=
\sup_{u\in \Sph^{n-1}}
\min_{c\in \Pcode_n(A)} \angles{u}{c}.
\]
Equivalently, $F_n(A)$ is the spherical covering radius of the finite code
$\Pcode_n(A)$. 
\end{definition}

\begin{definition}[Best achievable value at fixed alphabet size]
For $n,q\ge 2$, define
\[
w_{n,q}:=\inf\{F_n(A): A\in \mathcal A_q\}.
\]
\end{definition}

\begin{remark}[Scaling invariance]
\label{rem:scaling_invariance}
If $\lambda>0$, then
\[
\Pcode_n(\lambda A)=\Pcode_n(A)
\qquad\text{and}\qquad
F_n(\lambda A)=F_n(A),
\]
because
\[
\frac{\lambda x}{\|\lambda x\|_2}=\frac{x}{\|x\|_2}
\qquad (x\neq 0).
\]
Thus only the relative ratios inside $A$ matter geometrically.
\end{remark}

\begin{definition}[Optimal spherical covering radius]
For $n\ge 2$ and $m\ge 1$, define
\[
\rho_{\mathrm{sph}}(n,m)
:=
\inf\{
\covrad(C): C\subset \Sph^{n-1},\ |C|=m
\},
\]
where
\[
\covrad(C):=
\sup_{u\in \Sph^{n-1}}\min_{c\in C}\angles{u}{c}.
\]
\end{definition}

\section{Two-dimensional classification}
\label{sec:2d-class}
First, angular coverage is classified in two dimensions, analyzing the limitations of product codes relative to spherical codes in the minimal setting where this is possible.

\begin{lemma}[Exact spherical optimum on the circle]
\label{prop:circle}
For every $m\ge 1$,
\[
\rho_{\mathrm{sph}}(2,m)=\frac{\pi}{m}.
\]
More generally, if $C\subset \Sph^1$ has $m$ points, then
\[
\covrad(C)\ge \frac{\pi}{m},
\]
with equality if and only if the points of $C$ are equally spaced.
\end{lemma}
\begin{proof}
    \renewcommand{\qedsymbol}{}
    See Appendix \ref{app:2d-classification}.
\end{proof}

\begin{theorem}[Complete dimension-$2$ classification]
\label{thm:dim2-classification}
Let $A\subset \R$ be finite with $q:=|A|\ge 2$. Then exactly one of the
following holds.
\begin{enumerate}[label=(\roman*),nosep]
    \item $A=\{-a,a\}$ for some $a>0$, and then
    \[
    F_2(A)=\rho_{\mathrm{sph}}(2,q^2)=\rho_{\mathrm{sph}}(2,4)=\frac{\pi}{4};
    \]
    \item $A\ne\{-a,a\}$ for every $a>0$, and then
    \[
    \rho_{\mathrm{sph}}(2,q^2)<F_2(A).
    \]
\end{enumerate}
\end{theorem}

\begin{proof}
If $A=\{-a,a\}$, then $q=2$ and \cref{prop:binary-case} 
gives
\[
F_2(A)=\frac{\pi}{4}=\rho_{\mathrm{sph}}(2,4)=\rho_{\mathrm{sph}}(2,q^2).
\]

Now assume $A\ne\{-a,a\}$ for every $a>0$. If one of the three conditions in
\cref{lem:dim2-collision} holds, we are done by that lemma. Otherwise,
\cref{lem:no-collision-case} shows that $A=\{-u,v\}$ for some $u,v>0$, and our
assumption implies $u\ne v$. Then \cref{prop:binary-case} gives
\[
F_2(A)>\frac{\pi}{4}=\rho_{\mathrm{sph}}(2,4)=\rho_{\mathrm{sph}}(2,q^2).
\]
This proves the strict case.
\end{proof}

\begin{remark}[Exact-bit consequence in dimension $2$]
\label{cor:dim2-exact-bits}
Let $b\ge 1$ and let $A\subset \R$ be finite with $|A|=2^b$.
\begin{enumerate}[label=(\roman*),nosep]
    \item If $b\ge 2$, then
    \[
    \rho_{\mathrm{sph}}\bigl(2,2^{2b}\bigr)<F_2(A).
    \]
    \item If $b=1$, then equality can occur only for the antipodal binary
    alphabets $A=\{-a,a\}$.
\end{enumerate}
\end{remark}

The two-dimensional classification shows that, except for $A =\{-a, a\},\; \log_2|A| = 1$ (antipodal binary case), product codes are strictly suboptimal relative to spherical codes under the covering objective. Therefore, as bit-width increases beyond 1 bit per dimension, a spherical code is always superior. This establishes a clear separation between structured and unconstrained constructions at minimal dimensionality, and serves as a concrete illustration of the geometric limitations that persist and become more pronounced in higher dimensions.

\section{A harmonic witness lower bound}
\label{sec:harmonic_lb}
In this section, an explicit construction of a direction using harmonic numbers is presented, based on~\cite{Devore1998}. This direction is hard to cover in product code direction sets, resulting in an analytical lower bound for the product code covering objective. More specifically, we demonstrate that this direction, referred to as the \textit{harmonic witness} in the rest of this article, induces a lower bound on worst-case angular error.

\begin{definition}[Harmonic witness]
\label{def:harmonic_witness}
For $n\ge 1$, define the harmonic number
\[
H_n:=\sum_{i=1}^n \frac1i,
\]
and the unit vector
\[
u^{(n)}:=\left(\frac1{\sqrt{H_n}},\frac1{\sqrt{2H_n}},\dots,
\frac1{\sqrt{nH_n}}\right)\in \Sph^{n-1}.
\]
\end{definition}

\begin{definition}[Sign counts]
\label{def:sign-counts}
For a finite alphabet $A\subset \R$, define
\begin{align*}
p_+(A)
&:= |A\cap (0,\infty)|,\\
p_-(A)
&:= |A\cap (-\infty,0)|,\\
m(A)
&:= \min\{p_+(A),p_-(A)\}.
\end{align*}
\end{definition}

\begin{theorem}[Sign-count bound]
\label{cor:sign-symmetric}
For every finite alphabet $A\subset \R$ and every $n\ge 2$,
\[
F_n(A)
\ge
\arccos\!\left(\min\left\{1,2\sqrt{\frac{m(A)}{H_n}}\right\}\right),
\]
where $m(A)=\min\{p_+(A),p_-(A)\}$. In particular, if $A\setminus\{0\}$ has a
single sign, then
\[
F_n(A)\ge \frac{\pi}{2}.
\]
\end{theorem}
\begin{proof}
\renewcommand{\qedsymbol}{}
    See Appendix \ref{app:harmonic_witness}.
\end{proof}

\begin{corollary}[Uniform $q$-element consequence]
\label{cor:q-uniform-harmonic}
For every $q\ge 2$, every alphabet $A\in \mathcal A_q$, and every $n\ge 2$,
\[
F_n(A)
\ge
\arccos\!\left(\min\left\{1,2\sqrt{\frac{\lfloor q/2\rfloor}{H_n}}\right\}\right).
\]
\end{corollary}

The argument in this section has isolated a structural limitation of product codes: even under optimal choice of alphabet, there exist directions that cannot be well-aligned with any product-structured vector for large enough dimension to make the bound nontrivial. This bound forms the basis for the asymptotic separation established in the following section, where it is shown that spherical codes can achieve strictly better angular coverage in high dimensions.

\section{Asymptotic separation between product codes and spherical codes}
\label{sec:asym-sphere}
Using the lower bound proven in Section~\ref{sec:harmonic_lb}, this section shows that spherical codes are strictly better than product codes under our covering objective for any fixed alphabet size. 

\begin{definition}[Spherical cap covering number]
For $n\ge 2$ and $\theta\in (0,\pi)$, let $M_c(n,\theta)$ denote the minimum
number of spherical caps of angular radius $\theta$ required to cover
$\Sph^{n-1}$.
\end{definition}

We will make use of the following classical result\footnote{We have also formalized this result in Lean so our overall Lean proof is end-to-end and does not depend on this result as an axiom.}.

\begin{theorem}[Wyner {\cite{wyner-sphere-packing}}]
\label{thm:wyner}
Fix $\theta\in (0,\pi/2)$. Then
\[
M_c(n,\theta)=\exp\bigl(-n\log(\sin\theta)+o(n)\bigr)
\qquad (n\to\infty).
\]
\end{theorem}
\begin{proof}
    See~\cite{wyner-sphere-packing}.
\end{proof}

\begin{corollary}[Exponential spherical upper bound]
\label{cor:wyner-upper}
Let $\theta\in (0,\pi/2)$ and let $\lambda>1$. If
\[
\lambda\sin\theta>1,
\]
then there exists $N$ such that for all $n\ge N$,
\[
\rho_{\mathrm{sph}}\bigl(n,\lfloor \lambda^n\rfloor\bigr)\le \theta.
\]
\end{corollary}
\begin{proof}
\renewcommand{\qedsymbol}{}
See Appendix \ref{app:asymptotic_separation}.
\end{proof}

The total number of possible symbols stored by a product code in $n$ dimensions using an alphabet of $q$ elements is $q^n$. So the natural comparison for spherical code coverage is a spherical
code with the same number of symbols. This section has thus far proven that this spherical code has asymptotically improving coverage. We may now combine that result with the bound on possible coverage for the corresponding product code derived in the previous section.

\begin{theorem}[Asymptotic strict separation for fixed alphabet size]
\label{thm:asymptotic-strict}
Fix $q\ge 2$. Then there exists $N(q)$ such that for all $n\ge N(q)$ and all
alphabets $A\in \mathcal A_q$,
\[
\rho_{\mathrm{sph}}(n,q^n)<F_n(A).
\]
Hence also
\[
\rho_{\mathrm{sph}}(n,q^n)<w_{n,q}.
\]
\end{theorem}

\begin{proof}
Choose $\theta$ such that
\[
\arcsin\!\left(\frac1q\right)<\theta<\frac{\pi}{2}.
\]
Then $q\sin\theta>1$. By \cref{cor:wyner-upper}, there exists $N_1(q)$ such
that for all $n\ge N_1(q)$,
\[
\rho_{\mathrm{sph}}(n,q^n)\le \theta.
\]

On the product side, \cref{cor:q-uniform-harmonic} shows that for every
$A\in \mathcal A_q$,
\[
F_n(A)
\ge
\arccos\!\left(\min\left\{1,2\sqrt{\frac{\lfloor q/2\rfloor}{H_n}}\right\}\right).
\]
Because $H_n\to\infty$, the right-hand side tends to $\pi/2$. Since
$\theta<\pi/2$, there exists $N_2(q)$ such that for all $n\ge N_2(q)$,
\[
\arccos\!\left(\min\left\{1,2\sqrt{\frac{\lfloor q/2\rfloor}{H_n}}\right\}\right)>\theta.
\]
Therefore, for all $n\ge N_2(q)$ and all $A\in \mathcal A_q$,
\[
F_n(A)>\theta.
\]

Hence for all $n\ge N(q):=\max\{N_1(q),N_2(q)\}$ and all $A\in \mathcal A_q$,
\[
\rho_{\mathrm{sph}}(n,q^n)\le \theta < F_n(A),
\]
which is the desired strict separation. Since the inequality is uniform over
$A\in \mathcal A_q$, taking the infimum over all such alphabets yields the
statement for $w_{n,q}$.
\end{proof}

In particular, for the same storage budget in bits, $b$, the spherical code outperforms the product code:

\begin{remark}[Exact-bit consequence]
\label{cor:asymptotic-exact-bits}
Fix $b\ge 1$ and set $q:=2^b$. Then there exists $N(b)$ such that for all
$n\ge N(b)$ and all finite alphabets $A\subset \R$ with $|A|=2^b$,
\[
\rho_{\mathrm{sph}}\bigl(n,2^{bn}\bigr)<F_n(A).
\]
Hence also
\[
\rho_{\mathrm{sph}}\bigl(n,2^{bn}\bigr)<w_{n,2^b}.
\]
\end{remark}
~\\

Note that the dimension-$2$ and asymptotic results have different flavors. In dimension $2$, strict separation fails exactly for the antipodal binary alphabets $A=\{-a,a\}$. In contrast, for every fixed alphabet size $q\ge 2$, asymptotic strict separation holds. 

Thus we have shown strict suboptimality of product codes in our covering objective, relative to spherical codes (for large enough dimension), and quantified the gap as the difference between the product code lower bound (\cref{cor:q-uniform-harmonic}) and the spherical code upper bound (\cref{cor:wyner-upper}).

\section{Separation among product codes}
\label{sec:asymptotic-sep-prod-code}

This section analyzes the properties of product codes constructed by the specific commonly used alphabets of floating-point, fixed-point, and two's-complement. We prove the suboptimality of these formats and quantify the optimality gap versus a product code using an optimal alphabet. These three formats are referred to as the \textit{standard formats} henceforth. 

Importantly for all product codes, the covering objective tends to at least $\pi/2$ as $n\to\infty$, which is stated formally and proven in \cref{prop:prod-code-limit}. Therefore, unnormalized angular separation is not informative in the asymptote. However, we show in this section that a normalized version of coverage, which we refer to as the normalized orthogonality deficit, is informative and allows us to quantify the coverage cost of using standard number formats. Specifically, we study the quantity below, where the harmonic numbers are as in \cref{def:harmonic_witness}.

For convenience, in this section we will work with a complementary form of the covering objective:
\[
\alpha_n(A):=\inf_{u\in\Sph^{n-1}}
\max_{x\in A^n\setminus\{0\}}\frac{\langle u,x\rangle}{\|x\|_2},
\]
so
\[
F_n(A)=\arccos \alpha_n(A),
\qquad
\alpha_n(A)=\cos F_n(A).
\]
Since $\cos$ is decreasing on $[0,\pi]$, we also have
\[
\cos w_{n,q}
=
\sup\{\alpha_n(A):A\subset\R,\ |A|=q\}.
\]
\begin{definition}[Canonical floating-point family]
\label{def:canonical-float}
Fix a bit-width $b\ge2$. We use one sign bit, $e\ge1$ exponent bits, and
$t\ge0$ trailing mantissa bits ({\em i.e.} not counting the `hidden 1'), with
\[
b=1+e+t.
\]
With the smallest exponent denoting denormals but no special values (Inf, NaN), up to global scaling, define the positive decoded levels by
\begin{align*}
\Phi_{e,t}^+
&:=
\{1,2,\dots,2^t-1\}\\
&\quad \cup
\left\{
\begin{aligned}
(2^t+m)2^j:\;& 0\le m\le 2^t-1,\\
& 0\le j\le 2^e-2
\end{aligned}
\right\}.
\end{align*}
When $t=0$, the first set (denormals) is empty. The full decoded real alphabet is
\[
\Phi_{e,t}:=\{0\}\cup\Phi_{e,t}^+\cup(-\Phi_{e,t}^+).
\]

The product code covering objective for the floating-point alphabet family is denoted by $F^{fp}_{n,b}$.
\end{definition}

\begin{remark}[Fixed-point alphabets and two's complement]
\label{rem:fixed-point-twos-complement}
    The definition of the canonical floating-point family includes 0-symmetric fixed-point alphabets via the case $e=1$ (the denormals and the first and only binade of normal numbers). Furthermore, \cref{lem:one_unmatched_scalar} proves that two's complement alphabets offer no better angular coverage than symmetric fixed-point. Therefore, to prove suboptimality of the standard formats, it suffices to prove suboptimality of the floating-point alphabet family.
\end{remark}

\begin{lemma}[Fixed sign-symmetric alphabets]
\label{prop:fixed-sign-symmetric}
Let
\[
A=\{0\}\cup\pm\{c_1,\dots,c_m\},
\;\;\;
c_1>c_2>\cdots>c_m>0.
\]
Then, for every $n$,
\[
\sqrt{H_n}\,\alpha_n(A)
\le
2\sqrt{
1+\sum_{j=1}^{m-1}\frac{c_j-c_{j+1}}{c_j+c_{j+1}}
}.
\]
\end{lemma}
\begin{proof}
\renewcommand{\qedsymbol}{}
    See Appendix \ref{app:prod-code-sep}.
\end{proof}

\begin{corollary}[Floating-point normalized obstruction]
\label{cor:fp-obstruction}
For every $b\ge2$,
\[
\limsup_{n\to\infty}\sqrt{H_n}\cos F^{\fp}_{n,b}
\le
2\sqrt{\frac{2^{b-1}+1}{3}}.
\]
\end{corollary}

\begin{proof}
Let $m=2^{b-1}-1$. Write the positive levels of $\Phi_{e,t}$ in decreasing
order as $c_1>\cdots>c_m>0$. By \cref{lem:fp-ratio}, $c_j/c_{j+1}\le2$, and so
\[
\frac{c_j-c_{j+1}}{c_j+c_{j+1}}\le\frac13.
\]
\Cref{prop:fixed-sign-symmetric}
therefore gives
\begin{align*}
\sqrt{H_n}\,\alpha_n(\Phi_{e,t})
&\le 2\sqrt{1+\frac{m-1}{3}}\\
&= 2\sqrt{\frac{m+2}{3}}\\
&= 2\sqrt{\frac{2^{b-1}+1}{3}}.
\end{align*}
Since $F^{\fp}_{n,b}$ is the minimum over finitely many floating-point splits,
its cosine is the maximum of the corresponding $\alpha_n(\Phi_{e,t})$ values.
This proves the claim.
\end{proof}

Now that the bound is proven for the standard formats, we move to the complementary bound for arbitrary alphabets. The result is simply stated here, with the proof given in the appendix.

\begin{theorem}[Arbitrary alphabets]
\label{cor:arb-lower}
For every $b\ge2$,
\[
\liminf_{n\to\infty}\sqrt{H_n}\cos w_{n,2^b}
\ge
2\sqrt{2^{b-1}-1}.
\]
\end{theorem}
\begin{proof}
\renewcommand{\qedsymbol}{}
    See Appendix \ref{app:prod-code-sep}.
\end{proof}

Finally, we combine the previous results to prove the suboptimality of the standard formats.

\begin{theorem}[Superiority of arbitrary alphabets]
\label{thm:main}
For every bit-width $b\ge3$,
\begin{align*}
\liminf_{n\to\infty}\sqrt{H_n}\cos w_{n,2^b}
&\ge 2\sqrt{2^{b-1}-1}\\
&> 2\sqrt{\frac{2^{b-1}+1}{3}}\\
&\ge \limsup_{n\to\infty}\sqrt{H_n}\cos F^{\fp}_{n,b}.
\end{align*}
Thus the best arbitrary $2^b$-symbol scalar product alphabets converge to
$\pi/2$ strictly more slowly than the best $b$-bit floating-point alphabets on
the natural $1/\sqrt{H_n}\sim1/\sqrt{\log n}$ scale.
\end{theorem}

\begin{proof}
The first inequality is \cref{cor:arb-lower}. The last inequality is
\cref{cor:fp-obstruction}. It remains only to compare the constants. Since
$b\ge3$, we have $2^{b-1}-1>1$, and hence
\[
2^{b-1}-1>\frac{2^{b-1}+1}{3}.
\]
Taking square roots and multiplying by $2$ gives the strict middle inequality.
\end{proof}

\begin{remark}[The four-bit case]
\label{rem:4-bit-prod-code-sep}
For $b=4$,
\cref{thm:main} gives
\[
\liminf_{n\to\infty}\sqrt{H_n}\cos w_{n,16}
\ge
2\sqrt7,
\]
whereas
\[
\limsup_{n\to\infty}\sqrt{H_n}\cos F^{\fp}_{n,4}
\le
\sqrt{12}.
\]
The corresponding angular-deficit ratio is at least
\[
\frac{2\sqrt7}{\sqrt{12}}=\sqrt{\frac73}\approx1.528.
\]
So the normalized asymptotic separation is not merely strict; it has a concrete
constant-factor gap.
\end{remark}

This section focused within the class of product codes by comparing standard scalar formats to the best achievable alphabets under the covering objective. The results show that conventional constructions, including floating-point, fixed-point, and two’s complement representations, are asymptotically suboptimal when measured against an optimal product code of the same bit-width. Although all such formats tend to at least the same limit of the covering objective as dimensionality grows, they do so at strictly worse rates, as quantified in \cref{thm:main} under the normalization we propose. 

This establishes that the choice of scalar alphabet remains a critical factor even within the inherent limitations of product structure, and motivates the search for improved formats that better approximate the optimal directional geometry.

\section{Summary of theoretical results}
The preceding sections establish a theoretical framework for analyzing the angular coverage properties of product-structured codes.

We began with a two-dimensional classification. In this setting, product codes are shown to be suboptimal relative to spherical codes, except in the antipodal alphabet case in which both constructions operate under a total bit budget of two bits and are equivalent. For any larger bit budget, product codes are strictly suboptimal.

The two-dimensional analysis serves as the minimal non-trivial setting in which such a comparison can be made. Combined with the asymptotic analysis presented in Section~\ref{sec:asym-sphere}, this work provides a unified view spanning both low- and high-dimensional regimes. 

The asymptotic case proceeded by constructing a class of directions, the harmonic witnesses, that are provably difficult to represent using product codes. This construction yields a lower bound on the covering performance of product codes. In parallel, an application of Wyner's theorem~\cite{wyner-sphere-packing} provides an upper bound on the covering performance of spherical codes. In the high-dimensional limit, these bounds together establish the existence of a strict gap between the achievable angular coverage of spherical codes and that of any product code. Consequently, product codes are fundamentally suboptimal for angular coverage.

Despite this limitation, for reasons of efficiency, product codes remain the dominant paradigm for vector representation in modern machine learning systems. This motivates the question of whether improved alphabet constructions exist within the product code framework. Section~\ref{sec:asymptotic-sep-prod-code} answers this by quantifying the extent to which the standard formats of floating-point, two's complement, and symmetric fixed-point are suboptimal relative to an optimal product code for the same bit budget. 
\section{Experiments and Results}
\label{sec:experiments}
To complement the theoretical findings, experiments are conducted with product codes in higher dimensions, and empirical evidence is used to compare their angular coverage.

\begin{table}[t]
\caption{Sampled worst-case angular error for 4-bit alphabets across various dimensions (degrees).}
\label{tab:direction_error_dim32}
\centering
\footnotesize
\resizebox{\columnwidth}{!}{
\begin{tabular}{lccccc}
\hline
Datatype & $d=4$ & $d=8$ & $d=16$ & $d=32$ & $d=64$ \\
\hline
\bf{Optimized Alphabet} & \textbf{4.08} & \textbf{5.15} & \textbf{6.07} & \textbf{6.96} & \textbf{7.13} \\
\bf{E2M1} & 5.45 & 5.91 & 6.67 & 7.10 & 7.60 \\
\bf{INT/E1M2} & 6.75 & 8.08 & 9.58 & 10.9 & 11.4 \\
\bf{E3M0} & 10.7 & 10.9 & 11.0 & 11.1 & 11.1 \\
\hline
\end{tabular}
}
\end{table}
\begin{figure}[t]
    \centering
    \label{fig:p99-all-dtypes}
        \includegraphics[width=.9\columnwidth]{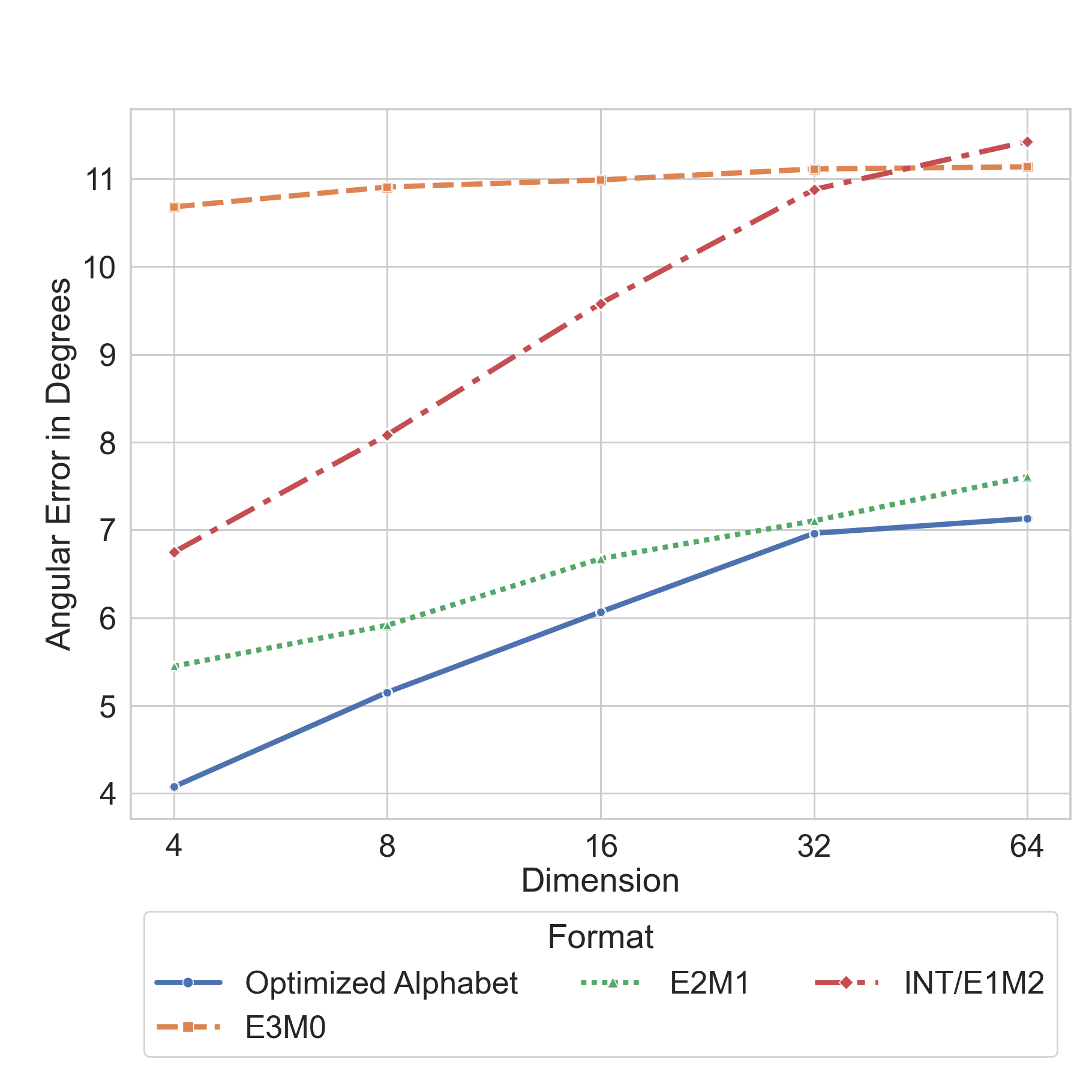}
    \caption{Sampled worst-case angular error with increasing dimension. E2M1, the format used in NVFP4, comes close to the optimized alphabet.}
    \label{fig:angular_error}
\end{figure}
\subsection{Quantifying angular error}
\label{sec:quantifying-angular-error}
As exact evaluation of the worst-case angular error (as measured by the covering objective in \cref{def:prod-code-objective}) over the sphere is intractable, sampling is used to estimate the covering objective. One million unit vectors are randomly generated following the method described in~\cite{sphere-sampling}, and mapped to their nearest codeword in the product code direction set under the angular distance metric. The largest angular error found is recorded.

In order to obtain a high-performance search for the nearest codeword, an extension of the method described in Lemma 3.1 of Gao \textit{et al.}~\cite{min_codeword_alg} is implemented. Specifically, to find the minimum angle between an arbitrary unit vector and a codeword from the product direction set, it is sufficient to perform component-wise nearest-neighbor quantization, but optimizing over all possible vector scales. This effectively reduces the dimensionality of the search to a single dimension (the scale), radically reducing the cost of this experimental search. This method is proven to find the nearest-angle codeword as formalized in \cref{thm:scale-min-codeword}.

All experiments are conducted up to dimension $d=64$. This dimension reflects a balance between computational tractability and practical relevance, as $d=64$ is double the standard block size for micro-scaling formats~\cite{ocp_mx} and $4\times$ that used in NVFP4. Experiments are focused on 4-bit alphabets for similar reasons: 4-bit formats are used in frontier low-precision machine learning models~\cite{nemotron_fp4}, and the time required for the optimization experiments discussed below scales with the number of alphabet elements, which is itself exponential in their bit-width.

\subsection{Optimization of Scalar Alphabets}
\label{sec:optimisation_alphabet}
Given that earlier results have shown optimal scalar alphabets can outperform standard scalar alphabets as the basis of a product code, it is natural to try to optimize the individual elements of a scalar code to optimize directional coverage. We use a derivative-free optimization framework and embed the worst-case angular error computation (described in \cref{sec:quantifying-angular-error}) as an inner evaluation routine over a different, smaller set of randomly generated unit vectors. 

This work adopts a two-stage, hybrid global-local approach. In the first stage, Differential Evolution (DE)~\cite{de_paper} is applied. This algorithm is a derivative-free method well suited to coarse search. In the second phase, the best-performing candidate solution obtained from the DE stage is used to initialize Powell’s method~\cite{powell_paper}. Powell’s method is also a derivative-free algorithm that performs successive line minimization along a set of adaptively updated search directions. Therefore, it is well suited for local optimization and solution refinement. 

The alphabet is forced to be sign-symmetric, as this reduces the complexity of the optimization problem. Therefore, we only optimize over positive values. Furthermore, the inclusion of zero is forced for accurate representation of sparse matrices.

Thus, the optimizer learns seven distinct values, $A_+$ , which define a 4-bit sign-symmetric scalar format $A_+ \cup A_- \cup \{0\}$, where $A_- = -A_+$. The optimized values are shown in \cref{tab:optimised_code_dimensions}, after normalization such that the first nonzero value is one (as detailed in \cref{rem:scaling_invariance}, this does not change coverage). 

\begin{table}[t]
\caption{Optimized sign-symmetric 4-bit alphabet values across various dimensions.}
\label{tab:optimised_code_dimensions}
\centering
\resizebox{\columnwidth}{!}{
\begin{tabular}{l@{\hspace{1.5em}}l}
\hline
Dimension & \multicolumn{1}{c}{Optimized Alphabet} \\
\hline
$d=4$  & $\begin{array}{@{}rrrrrrr@{}}
1 & 2.21 & 3.62 & 5.23 & 7.25 & 9.50 & 11.7
\end{array}$ \\
$d=8$  & $\begin{array}{@{}rrrrrrrr@{}}
1 & 2.26 & 3.67 & 5.42  & 7.54 & 10.5 & 13.0
\end{array}$ \\
$d=16$ & $\begin{array}{@{}rrrrrrrr@{}}
1 & 2.12 & 3.40  & 5.04 & 7.25 & 10.5   & 13.2
\end{array}$ \\
$d=32$ & $\begin{array}{@{}rrrrrrrr@{}}
1 & 2.41  & 3.86 & 5.57 & 6.12 & 7.78 & 15.6
\end{array}$ \\
$d=64$ & $\begin{array}{@{}rrrrrrrr@{}}
1 & 1.94 & 3.25  & 4.82 & 6.90 & 9.86 & 14.9
\end{array}$ \\
\hline
\end{tabular}
}
\end{table}
\subsection{Formats}
The experiments compare multiple 4-bit scalar formats consisting of typical formats alongside the optimized alphabet from the previous section:
\begin{itemize}
    \item \textbf{Optimized Alphabet}: Optimized scalar levels from the procedure discussed above.
    \item \textbf{E2M1}: 4-bit floating-point format with 2 exponent bits and 1 mantissa bit, as used in NVFP4~\cite{nvfp4_dtype}.
    \item \textbf{E3M0}: Consecutive powers of two, i.e. $A = \{0, \pm 2^{k}\}$ for $k \in \{0,\ldots,6\}$. 
    \item \textbf{INT/E1M2}: The two's complement integer representation in four bits. As shown in \cref{lem:one_unmatched_scalar}, this format has equivalent coverage to its fixed-point counterpart, which by \cref{def:canonical-float} is E1M2 (with denormals supported).
\end{itemize}
The worst-case angular error observed is recorded in Table~\ref{tab:direction_error_dim32}, and visualized in Figure~\ref{fig:angular_error}. The optimized alphabet performs the best across all dimensions, including $d=16$, which is the micro-block size used in NVFP4.

Among standard formats, it is observed that E2M1 obtains coverage close to that of the optimized alphabet, outperforming INT/E1M2 and E3M0. 

\subsection{Log-Space Analysis}

Figure~\ref{fig:log-log} shows log-log regressions between the optimized alphabet in $d=16$ and the other formats with fixed unit regression slope.

Unit slope on a log-log plot corresponds to an isometry, with the intercept corresponding to a multiplicative code scaling (see \cref{rem:scaling_invariance}).

It can be seen that E2M1 exhibits strong alignment with our independently optimized alphabet, indicating a near isometry.  E2M1 can thus be interpreted as a specific rounding of the optimized alphabet we propose to the best-fitting 4-bit floating-point family, trading a small reduction in directional fidelity for increased efficiency in computation.

\begin{figure}[htbp]
    \centering
    \subfloat{
        \includegraphics[width=0.9\columnwidth]{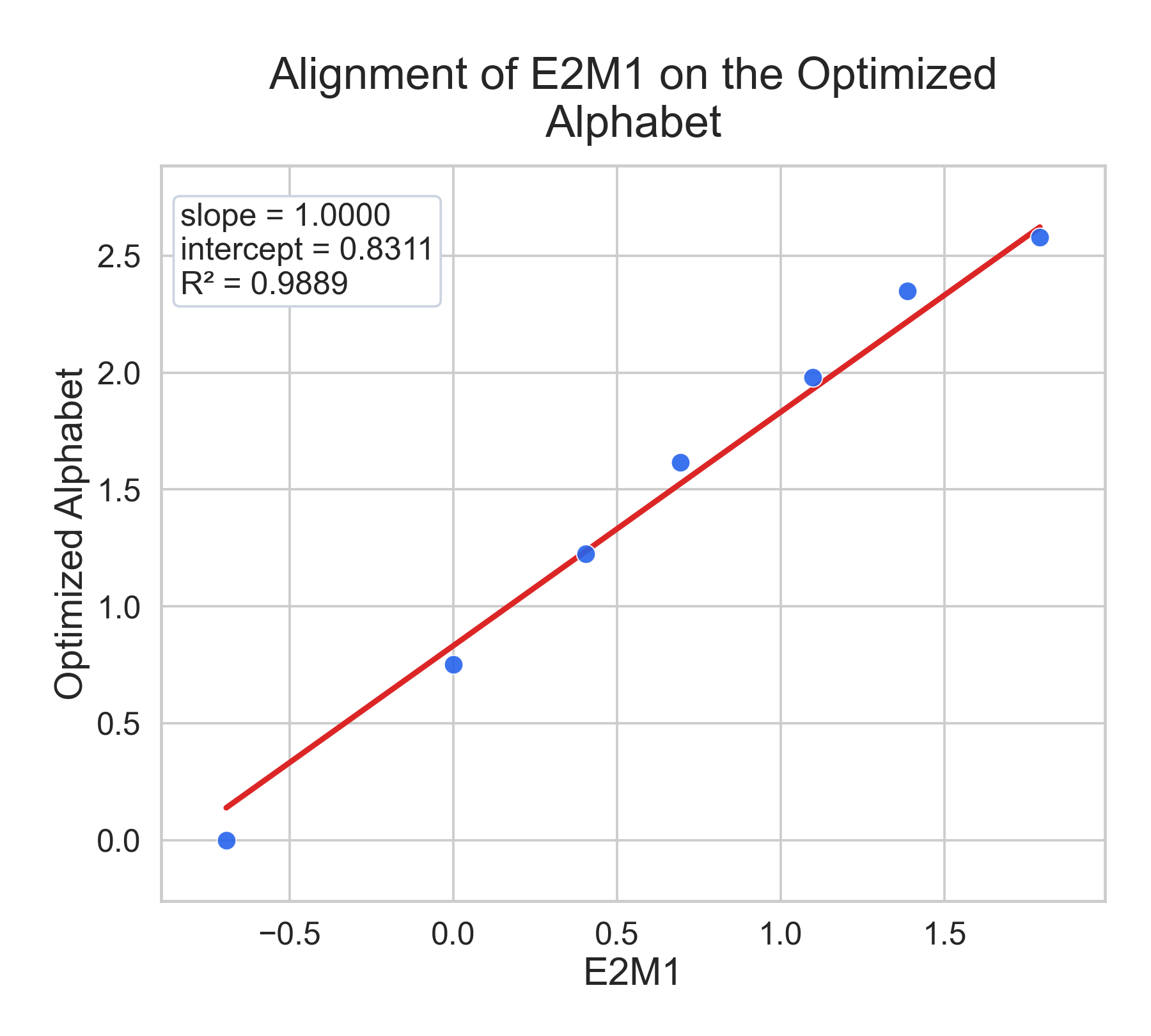}
    }
    \hfill
    \subfloat{
        \includegraphics[width=0.9\columnwidth]{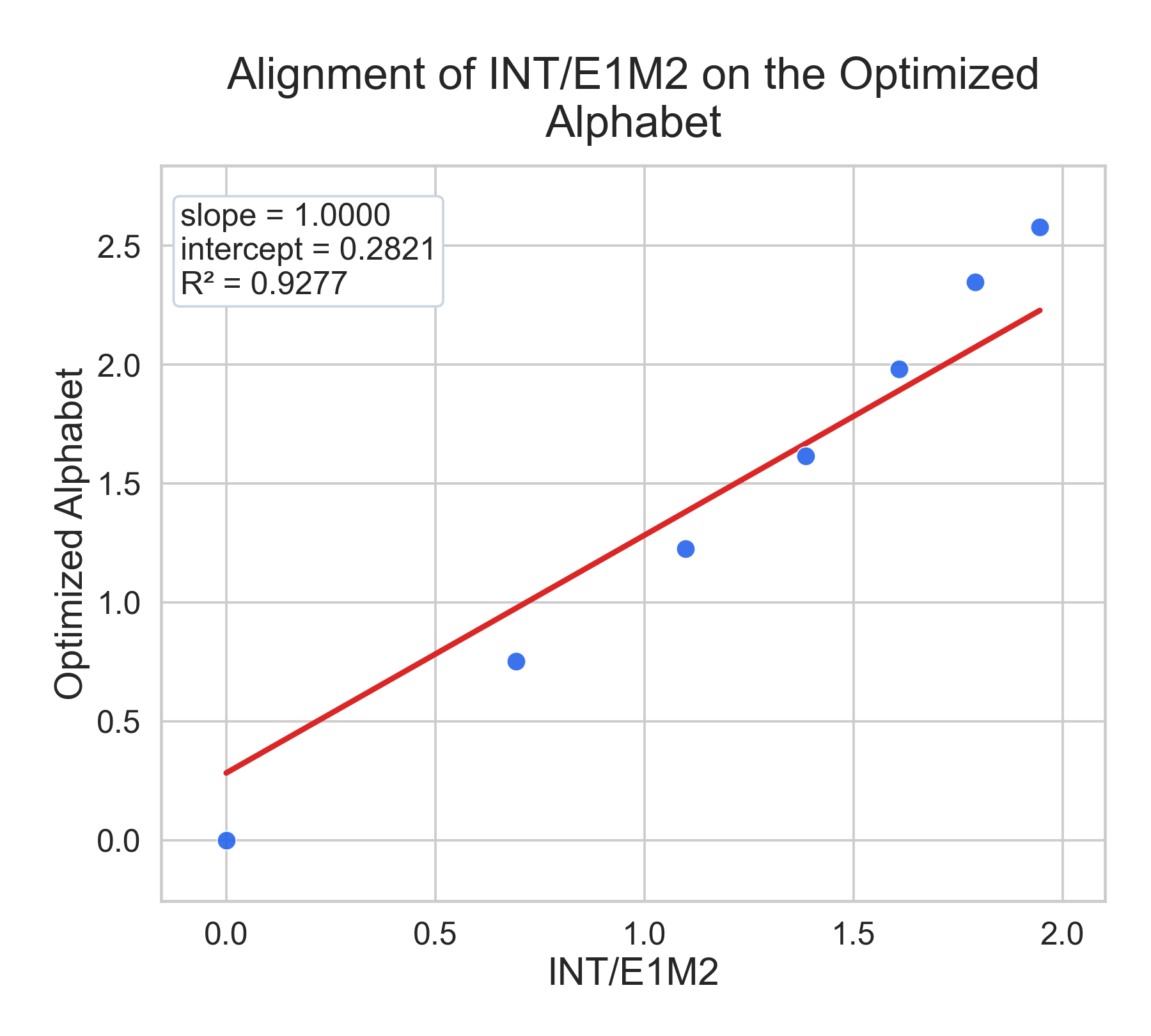}
    }
    \hfill
    \subfloat{
        \includegraphics[width=0.9\columnwidth]{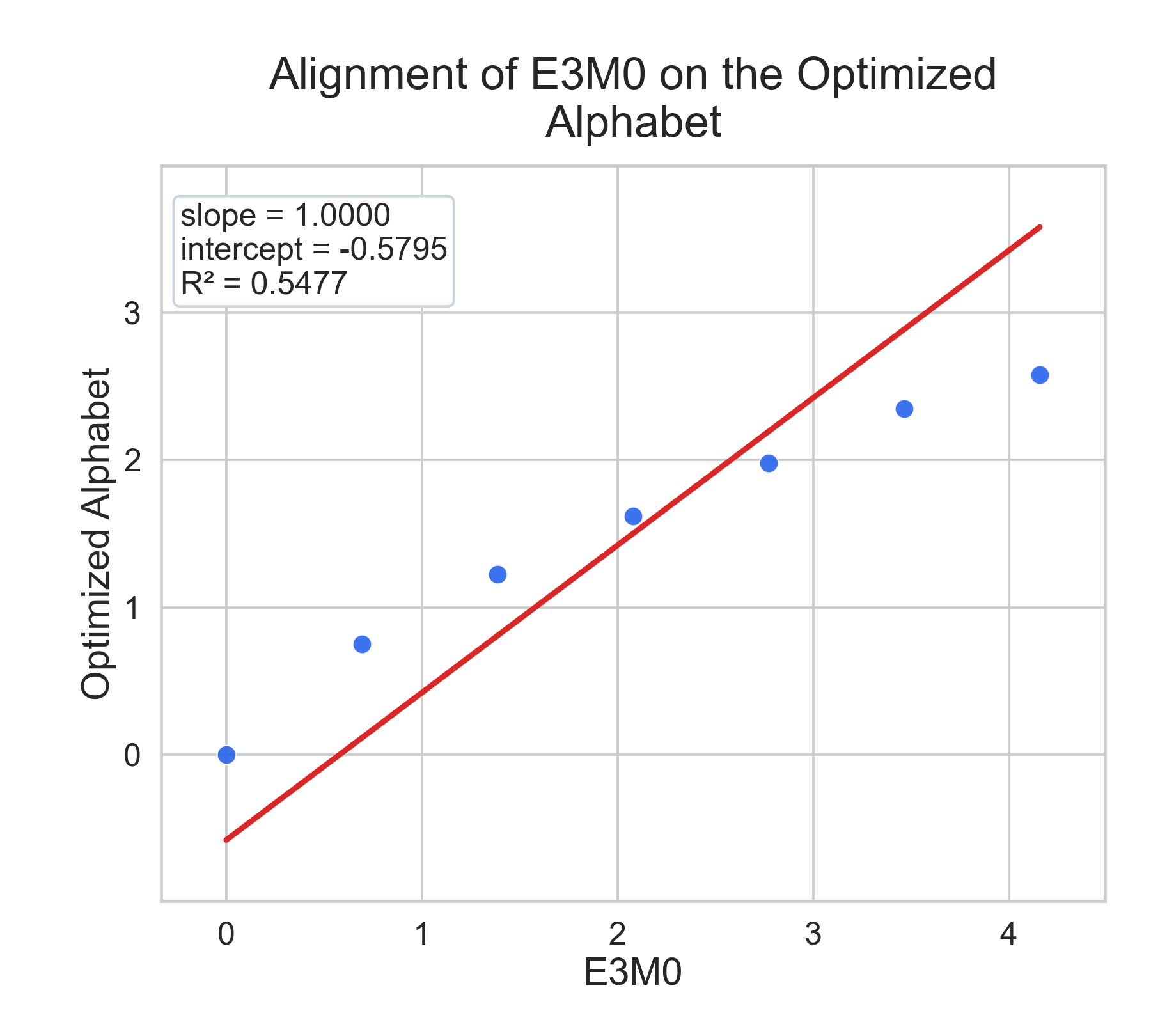}
    }
    \hfill
    \caption{Log-log regressions of various alphabets versus our optimized alphabet for 16 dimensions. The plots show that E2M1 has very strong alignment corresponding to a near isometry with our optimized code. The data point markers show the standard code's representable values.}
    \label{fig:log-log}
\end{figure}
\section{Conclusion}

This work studies the directional properties of product-structured vector codes and analyzes their ability to cover the unit sphere under a worst-case angular error metric, through both theory and experimentation.

The theoretical results establish a fundamental limitation: product-structured codes are inherently suboptimal, with quantified gap, for directional coverage in high dimensions and are at best equivalent in low dimensions. Furthermore, the standard formats of floating-point, two's complement, and fixed-point are suboptimal relative to an arbitrary alphabet, again with quantified gap. Thus, this work motivates research into new scalar formats for direction preservation.

Following this insight, an optimized scalar alphabet is obtained through numerical optimization. Using this, we supported the theoretical results empirically by showing that the standard formats obtain worse coverage than our new alphabet. Notably, the widely used E2M1 format performs similarly to this optimized structure. These findings are explained through log-space analysis, which revealed that E2M1 is well-aligned with the optimized alphabet.

These results provide a geometric perspective on the design of low-precision formats. While scalar representations are fundamentally limited compared to unconstrained spherical codes, carefully designed formats can closely approximate the optimal structure within this class. In particular, the strong performance of E2M1 within modern machine learning quantization can be understood as a consequence of its alignment with the optimal product-induced geometry.

Overall, this work highlights directional coverage as a key consideration in the design and analysis of low-precision representations, and provides both theoretical limits and empirical guidance for future format design.
\section*{Acknowledgments}
The authors acknowledge the use of GPT-5.4 and 5.5 from OpenAI and Aristotle from Harmonic in the development of proofs and experiments for this paper.

\bibliographystyle{IEEEtran}
\bibliography{references}
\newpage\clearpage

\numberwithin{theorem}{section}
\numberwithin{lemma}{section}
\numberwithin{corollary}{section}
\numberwithin{definition}{section}
\numberwithin{remark}{section}

\renewcommand{\thelemma}{\Alph{section}\arabic{lemma}}
\setcounter{lemma}{0}
\renewcommand{\thecorollary}{\Alph{section}\arabic{corollary}}
\setcounter{corollary}{0}

\appendices 
\raggedbottom

\section{Two-dimensional classification}
\label{app:2d-classification}
Following the classification in the main paper, this appendix provides an elementary but complete characterization of the two-dimensional case, as this is the first dimension at which spherical and product codes diverge, and can be easily visualized. We begin with a lemma that formalizes the intuition that spherical codes are no worse than product codes for the same number of induced directions.

\begin{lemma}[Pointwise comparison with spherical codes]
\label{lem:pointwise-comparison}
For every $n\ge 2$ and every finite alphabet $A\subset \R$,
\[
\rho_{\mathrm{sph}}\bigl(n,|\Pcode_n(A)|\bigr)\le F_n(A).
\]
\end{lemma}
\begin{proof}
By definition, $\Pcode_n(A)$ is itself a spherical code in $\Sph^{n-1}$ with exactly
$|\Pcode_n(A)|$ points, and its covering radius is exactly $F_n(A)$. The claim
follows because $\rho_{\mathrm{sph}}(n,m)$ is the infimum of the covering
radii over all $m$-point spherical codes. Therefore, a product code may never outperform a spherical code in angular coverage.
\end{proof}

We will also need to use the fact that adding directions to a spherical code can never make coverage worse:

\begin{lemma}[Antitonicity in the code size]
\label{lem:rho-antitone}
Let $n\ge 2$ and $1\le m_1\le m_2$. Then
\[
\rho_{\mathrm{sph}}(n,m_2)\le \rho_{\mathrm{sph}}(n,m_1).
\]
\end{lemma}

\begin{proof}
Every $m_1$-point spherical code becomes an admissible competitor for the
$m_2$-point problem after adjoining any $m_2-m_1$ extra points. Adding points
cannot increase covering radius. Taking the infimum over all $m_1$-point codes
gives the claim.
\end{proof}

In the main body of the paper, we stated the result that characterizes the exact spherical coverage optimum in 2D. Its proof is here for completeness:

\begin{proof}[\textbf{Proof of \cref{prop:circle}} (page \pageref{prop:circle})]
Let 
\[C=\{c_1,\dots,c_m\}\subset \Sph^1\]
listed in cyclic order around the
circle.
Let $\gamma_1,\dots,\gamma_m$ be the angular gaps between consecutive
points, so
\[
\gamma_i\ge0
\qquad\text{and}\qquad
\sum_{i=1}^m \gamma_i=2\pi.
\]
Hence
\[
\max_i \gamma_i\ge \frac{2\pi}{m}.
\]
The midpoint of a largest gap is at angular distance
$\frac12\max_i\gamma_i\ge \pi/m$ from every point of $C$, so
\[
\covrad(C)\ge \frac{\pi}{m}.
\]
Equality holds if and only if every gap equals $2\pi/m$, equivalently the
points are equally spaced.
Taking the infimum over all $m$-point sets $C$ shows that
$\rho_{\mathrm{sph}}(2,m)=\pi/m$.
\end{proof}

We now use a key feature of product codes to demonstrate one (combinatorial) obstruction to their efficient angular coverage: different codewords can represent the same angle:

\begin{lemma}[Collisions in dimension $2$]
\label{lem:dim2-collision}
Let $A\subset \R$ be finite with $q:=|A|\ge 2$. If at least one of the
following holds,
\begin{enumerate}[label=(\roman*),nosep]
    \item $0\in A$,
    \item $A$ contains two distinct positive elements,
    \item $A$ contains two distinct negative elements,
\end{enumerate}
then
\[
|\Pcode_2(A)|\le q^2-1.
\]
Consequently,
\[
F_2(A)>\frac{\pi}{q^2}=\rho_{\mathrm{sph}}(2,q^2).
\]
\end{lemma}

\begin{proof}
If $0\in A$, then $A^2\setminus\{0\}$ has only $q^2-1$ raw vectors, so
$|\Pcode_2(A)|\le q^2-1$.

If $a,b\in A$ are distinct positive elements, then the two raw vectors
$(a,a)$ and $(b,b)$ determine the same direction. Hence at least one collision occurs among the $q^2$ ordered
pairs in $A^2$, so again $|\Pcode_2(A)|\le q^2-1$.

The same argument applies if $A$ contains two distinct negative elements.

Now apply \cref{lem:pointwise-comparison,lem:rho-antitone} and
\cref{prop:circle}:
\begin{align*}
F_2(A)
&\ge \rho_{\mathrm{sph}}\bigl(2,|\Pcode_2(A)|\bigr)\\
&\ge \rho_{\mathrm{sph}}(2,q^2-1)\\
&= \frac{\pi}{q^2-1}\\
&> \frac{\pi}{q^2}.
\end{align*}
Since $\rho_{\mathrm{sph}}(2,q^2)=\pi/q^2$, the claimed strict inequality
follows.
\end{proof}

This leaves one remaining special case to characterize in two dimensions, which is covered over the next two lemmas.

\begin{lemma}[No-collision case]
\label{lem:no-collision-case}
Let $A\subset \R$ be finite with $|A|\ge 2$. If none of the three conditions in
\cref{lem:dim2-collision} holds, then
\[
A=\{-u,v\}
\qquad\text{for some }u,v>0.
\]
In particular, $|A|=2$.
\end{lemma}

\begin{proof}
If none of the three conditions holds, then $0\notin A$, $A$ contains at most
one positive element, and $A$ contains at most one negative element. Since
$|A|\ge 2$, it follows that $A$ has exactly one positive and exactly one
negative element, so
\[
A=\{-u,v\}
\qquad\text{for some }u,v>0.
\]
\end{proof}

\begin{lemma}[Binary no-collision case]
\label{prop:binary-case}
Let
\[
A=\{-u,v\}
\qquad (u,v>0).
\]
Then the following are equivalent:
\begin{enumerate}[label=(\roman*),nosep]
    \item $F_2(A)=\pi/4$,
    \item $\Pcode_2(A)$ is an equally spaced $4$-point code on $\Sph^1$,
    \item $u=v$.
\end{enumerate}
Consequently,
\begin{align*}
F_2(A)
&= \frac{\pi}{4}
&&\text{if } u=v,\\
F_2(A)
&> \frac{\pi}{4}
&&\text{if } u\ne v.
\end{align*}
\end{lemma}

\begin{proof}
The four raw vectors in $A^2$ are
\[
(-u,-u),\quad (-u,v),\quad (v,-u),\quad (v,v).
\]
These yield two antipodal pairs of directions:
\[
\pm \frac{(1,1)}{\sqrt2}
\qquad\text{and}\qquad
\pm \frac{(-u,v)}{\sqrt{u^2+v^2}}.
\]
Thus $\Pcode_2(A)$ is the union of the two lines
$\R(1,1)$ and $\R(-u,v)$ intersected with the unit circle.

A $4$-point subset of the circle has covering radius $\pi/4$ if and only if
its four points are equally spaced, by \cref{prop:circle}. For a union of two
antipodal pairs, equal spacing is equivalent to the two underlying lines being
orthogonal. Here orthogonality means
\[
(1,1)\cdot(-u,v)=0,
\]
that is,
\[
-u+v=0.
\]
So equal spacing holds exactly when $u=v$.

Therefore the three stated conditions are equivalent, and the final claim again
follows from \cref{prop:circle}.
\end{proof}

Note that, in particular, this antipodal special case only occurs with binary alphabets; spherical codes are strictly better beyond the binary case:

\begin{corollary}[Uniform dimension-$2$ strict separation for $q\ge 3$]
\label{cor:dim2-qge3}
For every $q\ge 3$ and every alphabet $A\in \mathcal A_q$,
\[
\rho_{\mathrm{sph}}(2,q^2)<F_2(A).
\]
Hence also
\[
\rho_{\mathrm{sph}}(2,q^2)<w_{2,q}.
\]

For $q=2$, equality occurs precisely for the antipodal alphabets
$A=\{-a,a\}$.
\end{corollary}
\begin{proof}
If $q\ge 3$, the exceptional case $A=\{-a,a\}$ is impossible, so the strict statement follows from \cref{thm:dim2-classification}. Since the inequality is uniform over $A\in\mathcal A_q$, taking the infimum over all such alphabets yields the statement for $w_{2,q}$. The last sentence is exactly the equality case in
\cref{thm:dim2-classification}.
\end{proof}

\section{A harmonic witness lower bound}
\label{app:harmonic_witness}

In this section, we derive the bound for the asymptotic case of large dimensions (for fixed alphabet). 

We will need to make use of a standard rearrangement inequality:

\begin{lemma}[Rearrangement inequality~\cite{HLP1952}]
\label{lem:weighted-sorting}
Let $y=(y_1,\dots,y_n)\in \R^n$ satisfy
\[
y_1\ge y_2\ge \cdots \ge y_n.
\]
For any $x=(x_1,\dots,x_n)\in \R^n$, let
\[
x^\downarrow=(x_1^\downarrow,\dots,x_n^\downarrow)
\]
denote the vector obtained by rearranging the coordinates of $x$ in
nonincreasing order. Then
\[
\sum_{i=1}^n y_i x_i
\le
\sum_{i=1}^n y_i x_i^\downarrow.
\]
\end{lemma}

\begin{proof}
See~\cite{HLP1952}.
\end{proof}

The partial sum of the coordinates of the harmonic witness from \cref{def:harmonic_witness} can be bounded as follows.
\begin{lemma}[Partial sums of the harmonic witness]
\label{lem:partial-sums}
For every integer $m\ge 1$,
\[
\sum_{i=1}^m \frac1{\sqrt{i}}\le 2\sqrt{m}.
\]
Consequently,
\[
\sum_{i=1}^m u^{(n)}_i \le 2\sqrt{\frac{m}{H_n}}
\qquad (1\le m\le n).
\]
\end{lemma}

\begin{proof}
For each $i\ge 1$,
\[
\frac1{\sqrt{i}}
\le
2\bigl(\sqrt{i}-\sqrt{i-1}\bigr),
\]
because this inequality is equivalent to
\[
1\le 2\sqrt{i}\bigl(\sqrt{i}-\sqrt{i-1}\bigr)
=
\frac{2\sqrt{i}}{\sqrt{i}+\sqrt{i-1}}.
\]
Summing from $i=1$ to $m$ yields
\[
\sum_{i=1}^m \frac1{\sqrt{i}}
\le
2\sum_{i=1}^m \bigl(\sqrt{i}-\sqrt{i-1}\bigr)
=
2\sqrt{m}.
\]
Dividing by $\sqrt{H_n}$ gives the second statement.
\end{proof}

This bound can, in turn, be used to place an upper bound on the inner product of the harmonic witness with vectors populated by elements drawn from a finite alphabet.

\begin{lemma}[Step-vector estimate]
\label{lem:step-vector}
Let $n\ge 1$, and let $z\in [0,\infty)^n$ be nonincreasing:
\[
z_1\ge z_2\ge \cdots \ge z_n\ge 0.
\]
Assume that $z$ takes at most $r\ge 0$ distinct positive values. Then
\[
\langle u^{(n)},z\rangle \le 2\sqrt{\frac{r}{H_n}}\,\|z\|_2.
\]
\end{lemma}

\begin{proof}
The proof is based on dividing such a vector into its level sets following~\cite{Burchard2009}.
If $r=0$, then $z=0$ and the claim is trivial. Assume $r\ge 1$. Let $s$ be the
actual number of distinct positive values taken by $z$, so $1\le s\le r$.
Because $z$ is nonincreasing, there exist positive numbers
\[
c_1>c_2>\cdots>c_s>0
\]
and positive integers $m_1,\dots,m_s$ such that
\[
z=
(\underbrace{c_1,\dots,c_1}_{m_1},
 \underbrace{c_2,\dots,c_2}_{m_2},
 \dots,
 \underbrace{c_s,\dots,c_s}_{m_s},
 0,\dots,0).
\]
Set $M_0:=0$ and $M_j:=m_1+\cdots+m_j$ for $1\le j\le s$, and define the
consecutive intervals
\[
I_j:=\{M_{j-1}+1,\dots,M_j\}
\qquad (1\le j\le s).
\]
Then $|I_j|=m_j$ and
\[
z_i=c_j
\qquad\text{for every }i\in I_j.
\]
Let
\[
d_j:=\sum_{i\in I_j} u^{(n)}_i.
\]
Then
\[
\langle u^{(n)},z\rangle = \sum_{j=1}^s c_j d_j = \sum_{j=1}^s (c_j\sqrt{m_j})\left(\frac{d_j}{\sqrt{m_j}}\right).
\]
Then by Cauchy--Schwarz,
\[
\left(\sum_{j=1}^s c_j d_j\right)^2
\le
\left(\sum_{j=1}^s c_j^2 m_j\right)
\left(\sum_{j=1}^s \frac{d_j^2}{m_j}\right).
\]

The first factor is exactly $\|z\|_2^2$. For the second factor, because
$u^{(n)}$ is nonincreasing and $I_j$ has length $m_j$, we have
\[
d_j\le \sum_{i=1}^{m_j} u^{(n)}_i \le 2\sqrt{\frac{m_j}{H_n}}
\]
by \cref{lem:partial-sums}. Therefore
\[
\frac{d_j^2}{m_j}\le \frac{4}{H_n}
\qquad (1\le j\le s),
\]
and hence
\[
\sum_{j=1}^s \frac{d_j^2}{m_j}\le \frac{4s}{H_n}\le \frac{4r}{H_n}.
\]
Combining the preceding inequalities gives
\[
\langle u^{(n)},z\rangle^2
\le
\|z\|_2^2\,\frac{4r}{H_n},
\]
which is equivalent to the stated bound.
\end{proof}

Finally, using the preceding lemmas, we prove the upper bound provided in the main paper for the best achievable covering objective for any product code with alphabet size $q$ in dimension $n$. 

\begin{proof}[\textbf{Proof of \cref{cor:sign-symmetric}} (page \pageref{cor:sign-symmetric})]
Let \(A\subset \mathbb{R}\) be finite, and set $p_+(A), p_-(A), m(A)$ as in \cref{def:sign-counts} (page \pageref{def:sign-counts}).

Choose \(\sigma\in\{-1,1\}\) so that
\[
|\{a\in A:\sigma a>0\}|=m(A).
\]
Thus, if the positive entries of \(A\) are fewer, take \(\sigma=1\);
if the negative entries are fewer, take \(\sigma=-1\).  Put
\(v=\sigma u^{(n)}\).

Fix \(x\in A^n\setminus\{0\}\).  Define
\[
y_i:=\max\{\sigma x_i,0\}.
\]
Then
\[
\langle v,x\rangle
=
\sum_{i=1}^n u_i^{(n)}\,\sigma x_i
\le
\sum_{i=1}^n u_i^{(n)} y_i.
\]
Moreover, \(y\) takes at most \(m(A)\) distinct positive values, and
\(\|y\|_2\le \|x\|_2\).

Let \(z\) be the nonincreasing rearrangement of \(y\).  Since
\(u^{(n)}\) is nonincreasing, the rearrangement inequality gives
\[
\sum_{i=1}^n u_i^{(n)} y_i
\le
\sum_{i=1}^n u_i^{(n)} z_i
=
\langle u^{(n)},z\rangle.
\]

The vector \(z\) is nonnegative, nonincreasing, and takes at most
\(m(A)\) distinct positive values.  Therefore \cref{lem:step-vector} gives
\begin{align*}
\langle u^{(n)},z\rangle
&\le 2\sqrt{\frac{m(A)}{H_n}}\|z\|_2\\
&= 2\sqrt{\frac{m(A)}{H_n}}\|y\|_2\\
&\le 2\sqrt{\frac{m(A)}{H_n}}\|x\|_2.
\end{align*}
Hence
\[
\frac{\langle v,x\rangle}{\|x\|_2}
\le
2\sqrt{\frac{m(A)}{H_n}}
\]
for every \(x\in A^n\setminus\{0\}\).

Thus every codeword
has inner product at most \(2\sqrt{m(A)/H_n}\) with the single unit
vector \(v\).  Therefore every codeword makes angle at least

\[
\arccos\left(
\min\left\{1,2\sqrt{\frac{m(A)}{H_n}}\right\}
\right)
\]
with \(v\).  Since \(F_n(A)\) is the supremum over all directions,
the desired lower bound follows.

This proves the first statement. If all nonzero alphabet values have the same
sign, then $m(A)=0$, and the displayed formula reduces to
\[
F_n(A)\ge \arccos(0)=\frac{\pi}{2}.
\]
\end{proof}

\section{Asymptotic strict separation for every fixed alphabet size}
\label{app:asymptotic_separation}
In this section we show that using Wyner's result in \cite{wyner-sphere-packing}, an upper bound for the covering objective of a spherical code can be derived, for large enough $n$. When combined with the lower bound in \cref{cor:sign-symmetric} (proven in the previous appendix), we will then be able to obtain the claimed separation result between product and spherical codes, the key result on the spherical code side presented in the main paper.
\begin{proof}[\textbf{Proof of \cref{cor:wyner-upper}} (page \pageref{cor:wyner-upper})]
By \cref{thm:wyner}, for every $\varepsilon>0$ we have
\[
M_c(n,\theta)
\le
\exp\bigl(n(-\log(\sin\theta)+\varepsilon)\bigr)
\]
for all sufficiently large $n$. Since $\lambda\sin\theta>1$, we have
\[
\log\lambda> -\log(\sin\theta),
\]
so we may choose $\varepsilon>0$ such that
\[
-\log(\sin\theta)+\varepsilon<\log\lambda.
\]
Hence, for all sufficiently large $n$,
\[
M_c(n,\theta)<\lambda^n.
\]
Since $M_c(n,\theta)$ is an integer, this implies
\[
M_c(n,\theta)\le \lfloor \lambda^n\rfloor.
\]
Therefore there exists a $\lfloor \lambda^n\rfloor$-point spherical code with
covering radius at most $\theta$, equivalently
\[
\rho_{\mathrm{sph}}\bigl(n,\lfloor \lambda^n\rfloor\bigr)\le \theta.
\]
\end{proof}

\section{Separation among product codes}
\label{app:prod-code-sep}
This section presents results supporting the demonstration in the main paper of the gap between the optimal coverage of a code based on floating-point, fixed-point, or two's complement alphabets, relative to an optimal product code.

We begin by proving the limit of the product code covering objective

\begin{lemma}[Product code coverage in the limit]
\label{prop:prod-code-limit}
Let $A \subset \mathbb{R}$ be finite and suppose that $A$ contains at least
one strictly positive element and at least one strictly negative element. Then
\[
\lim_{n\to\infty} F_n(A) = \frac{\pi}{2}.
\]

if $A$ contains values of only one sign, then
\[
\lim_{n\to\infty} F_n(A) \ge \frac{\pi}{2}.
\]
\end{lemma}

\begin{proof}
The second statement is obtained directly from the proof of \cref{cor:sign-symmetric}, which shows that the covering objective is greater than $\arccos(0) = \pi/2$ for single-sign alphabets. It remains to prove the first statement. 
Let $q := |A|$. By \cref{cor:q-uniform-harmonic}, for all $n \ge 2$,
\begin{align*}
F_n(A)
&\ge
\arccos\!\left(
\min\left\{1,\,2\sqrt{\frac{\lfloor q/2\rfloor}{H_n}}\right\}
\right).
\end{align*}
Since $H_n \to \infty$ as $n \to \infty$, we have
\begin{align*}
2\sqrt{\frac{\lfloor q/2\rfloor}{H_n}} &\to 0,
\end{align*}
and therefore, by continuity of $\arccos$,
\begin{align*}
\liminf_{n\to\infty} F_n(A)
&\ge
\arccos(0)
=
\frac{\pi}{2}.
\end{align*}

It remains to prove the matching upper bound. By assumption, there exist 
$a_+ \in A$ and $a_- \in A$ such that $a_+>0$ and $a_-<0$. Hence
\[
(a_+,\ldots,a_+) \in A^n,
\text{ and }
(a_-,\ldots,a_-) \in A^n.
\]
After normalization, these vectors give the antipodal pair
\[
c := \frac{1}{\sqrt{n}}(1,\ldots,1),
\text{ and }
-c
\]
both in $\Pcode_n(A)$. Therefore, for every $u\in \Sph^{n-1}$,
\begin{align*}
\min_{d\in \Pcode_n(A)} \angles{u}{d}
&\le
\min\{\angles{u}{c},\angles{u}{-c}\}
\\
&\le
\frac{\pi}{2}.
\end{align*}
Taking the supremum over $u\in\Sph^{n-1}$ yields
\[
F_n(A)\le \frac{\pi}{2}
\]
for every $n\ge 2$. Hence
\[
\limsup_{n\to\infty} F_n(A)\le \frac{\pi}{2}.
\]

Combining the lower and upper bounds gives
\[
\lim_{n\to\infty} F_n(A)=\frac{\pi}{2}.
\]
\end{proof}

In the main body of the paper, \cref{rem:fixed-point-twos-complement} explains that to prove suboptimality of the standard formats, it suffices to consider only floating-point families because symmetric fixed-point formats are included in that family via our definition. It remains to prove that the additional negative value provided by two's complement alphabets offers no improvement over symmetric fixed-point.

\begin{lemma}[Unmatched scalar]
\label{lem:one_unmatched_scalar}
Let \(n \ge 2\). Let \(B \subset \mathbb{R}\) be a finite alphabet such that
\[
0 \in B, \qquad B=-B, \qquad B \neq \{0\}.
\]
For any \(a \in \mathbb{R}\), define
\[
A := B \cup \{a\}.
\]
Then
\[
F_n(A)=F_n(B).
\]
\end{lemma}

\begin{proof}
If \(a \in B\), then \(A=B\) and the claim is immediate. Hence assume
\(a \notin B\). Since \(0 \in B\), this implies \(a \neq 0\).

Because \(B \subset A\), adding the extra alphabet value cannot increase the
covering radius, so
\[
F_n(A) \le F_n(B).
\]
It remains to prove the reverse inequality.

For a finite alphabet \(C \subset \mathbb{R}\), define
\[
\mu_C(u)
:=
\max_{x \in C^n \setminus \{0\}}
\frac{\langle u,x\rangle}{\|x\|_2},
\qquad u \in S^{n-1}.
\]
Then
\[
F_n(C)
=
\sup_{u \in S^{n-1}} \arccos \mu_C(u).
\]

We first note that, since \(B=-B\), the quantity \(\mu_B(u)\) is invariant
under coordinatewise sign changes of \(u\). In particular,
\[
\mu_B(u)=\mu_B(|u|),
\]
where \(|u|\) denotes the coordinatewise absolute value of \(u\). Indeed, if
\[
D_u := \operatorname{diag}(\operatorname{sgn}(u_1),\ldots,
\operatorname{sgn}(u_n)),
\]
with an arbitrary choice of sign when \(u_i=0\), then \(D_u B^n = B^n\), and
\[
\langle u,x\rangle
=
\langle |u|,D_u x\rangle.
\]
Taking maxima over \(x \in B^n \setminus \{0\}\) gives the claim. Therefore
\[
F_n(B)
=
\sup_{\substack{v \in S^{n-1}\\ v \ge 0}}
\arccos \mu_B(v).
\]

Let
\[
s := \operatorname{sgn}(a) \in \{-1,1\}.
\]
We claim that for every \(v \in S^{n-1}\) with \(v \ge 0\),
\[
\mu_A(-s v)=\mu_B(-s v).
\]
The inequality \(\mu_A(-s v) \ge \mu_B(-s v)\) follows from \(B \subset A\).
For the reverse inequality, fix \(x \in A^n \setminus \{0\}\). Construct
\(x' \in B^n\) by replacing every coordinate of \(x\) equal to the added value
\(a\) by \(0\), and leaving all other coordinates unchanged. Since the removed
coordinates contribute
\[
(-s v_i)a = -|a|v_i \le 0
\]
to the inner product, we have
\[
\langle -s v,x\rangle \le \langle -s v,x'\rangle,
\qquad
\|x'\|_2 \le \|x\|_2.
\]

Also, since \(B \neq \{0\}\) and \(B=-B\), the alphabet \(B\) contains both
\(b\) and \(-b\) for some \(b>0\). As \(v\) is a nonnegative unit vector,
\[
\mu_B(-s v) \ge 0.
\]

If \(x'=0\), then
\[
\langle -s v,x\rangle \le 0,
\]
and hence
\[
\frac{\langle -s v,x\rangle}{\|x\|_2}
\le 0
\le \mu_B(-s v).
\]
If \(x' \neq 0\), then either \(\langle -s v,x'\rangle \le 0\), in which case
the same conclusion follows, or \(\langle -s v,x'\rangle >0\), in which case
\[
\frac{\langle -s v,x\rangle}{\|x\|_2}
\le
\frac{\langle -s v,x'\rangle}{\|x\|_2}
\le
\frac{\langle -s v,x'\rangle}{\|x'\|_2}
\le
\mu_B(-s v).
\]
Thus every \(x \in A^n \setminus \{0\}\) satisfies
\[
\frac{\langle -s v,x\rangle}{\|x\|_2}
\le
\mu_B(-s v),
\]
so
\[
\mu_A(-s v) \le \mu_B(-s v).
\]
Therefore
\[
\mu_A(-s v)=\mu_B(-s v)
\]
for every \(v \in S^{n-1}\) with \(v \ge 0\).

Using this identity and the sign-symmetry of \(B\), we obtain
\[
\begin{aligned}
F_n(A)
&=
\sup_{u \in S^{n-1}} \arccos \mu_A(u) \\
&\ge
\sup_{\substack{v \in S^{n-1}\\ v \ge 0}}
\arccos \mu_A(-s v) \\
&=
\sup_{\substack{v \in S^{n-1}\\ v \ge 0}}
\arccos \mu_B(-s v) \\
&=
\sup_{\substack{v \in S^{n-1}\\ v \ge 0}}
\arccos \mu_B(v) \\
&=
F_n(B).
\end{aligned}
\]
Combining this with \(F_n(A)\le F_n(B)\) proves
\[
F_n(A)=F_n(B).
\]
\end{proof}

The technical lemma below on decreasing positive finite sequences will be needed in the later proof of \cref{prop:fixed-sign-symmetric}.
\begin{lemma}[Quadratic form]
\label{lem:quadratic-identity}
Let
\[
c_1>c_2>\cdots>c_m>0
\]
and let $Q\in\R^{m\times m}$ be defined by
\[
Q_{ij}:=c_{\max\{i,j\}}^2.
\]
Then $Q$ is positive definite and
\[
c^\top Q^{-1}c
=
1+
\sum_{j=1}^{m-1}\frac{c_j-c_{j+1}}{c_j+c_{j+1}},
\]
where $c=(c_1,\dots,c_m)^\top$.
\end{lemma}

\begin{proof}
For $d\in\R^m$, write $t_j:=d_1+\cdots+d_j$ with $t_0=0$. Then
\[
d^\top Qd
=
\sum_{j=1}^m c_j^2(t_j^2-t_{j-1}^2)
=
\sum_{j=1}^{m-1}(c_j^2-c_{j+1}^2)t_j^2+c_m^2t_m^2.
\]
The final expression is strictly positive for $d\ne0$, because the map
$d\mapsto(t_1,\dots,t_m)$ is invertible and all displayed coefficients are
positive. Hence $Q$ is positive definite.

Let $y=Q^{-1}c$ and set $S_i=y_1+\cdots+y_i$. The equation $Qy=c$ says
\[
c_i^2S_i+\sum_{j>i}c_j^2y_j=c_i
\qquad(1\le i\le m).
\]
Subtracting the equation for $i+1$ from that for $i$ gives
\[
(c_i^2-c_{i+1}^2)S_i=c_i-c_{i+1},
\]
so
\[
S_i=\frac1{c_i+c_{i+1}}
\qquad(1\le i<m).
\]
The final equation gives $S_m=1/c_m$. Summation by parts yields
\begin{align*}
c^\top y
&= \sum_{j=1}^m c_j y_j\\
&= \sum_{j=1}^{m-1}(c_j-c_{j+1})S_j + c_m S_m\\
&= 1 + \sum_{j=1}^{m-1}\frac{c_j-c_{j+1}}{c_j+c_{j+1}}.
\end{align*}
Since $y=Q^{-1}c$, this proves the identity.
\end{proof}

We are now in a position to prove \cref{prop:fixed-sign-symmetric} from the main body of the paper, which places an upper bound on the normalized orthogonality deficit of a product code. This is used later to prove the suboptimality result. The proof of this lemma is given below.
\begin{proof}[\textbf{Proof of \cref{prop:fixed-sign-symmetric}} (page \pageref{prop:fixed-sign-symmetric})]
It suffices to evaluate the code against the single harmonic witness direction $u^{(n)}$. Since
$u^{(n)}$ has positive nonincreasing coordinates, a maximizing vector may be
assumed nonnegative and sorted in nonincreasing order, by deleting negative
coordinates and applying the rearrangement inequality.

Thus we may write the positive coordinates of $x$ in blocks with levels
$c_1,\dots,c_m$. Let $k_j$ be the cumulative number of coordinates in the first
$j$ positive blocks, $k_0=0$, and set
\[
t_j:=\sqrt{k_j},
\qquad
d_j:=t_j-t_{j-1}.
\]
Using
\[
\sum_{i=k_{j-1}+1}^{k_j}\frac1{\sqrt i}
\le
2(\sqrt{k_j}-\sqrt{k_{j-1}})=2d_j,
\]
we get
\[
\langle u^{(n)},x\rangle
\le
\frac2{\sqrt{H_n}}\sum_{j=1}^m c_jd_j.
\]
On the other hand,
\[
\|x\|_2^2
=
\sum_{j=1}^m c_j^2(k_j-k_{j-1})
=
\sum_{j=1}^m c_j^2(t_j^2-t_{j-1}^2)
=
d^\top Qd,
\]
where $Q_{ij}=c_{\max\{i,j\}}^2$. Hence
\[
\sqrt{H_n}\frac{\langle u^{(n)},x\rangle}{\|x\|_2}
\le
2\frac{c^\top d}{\sqrt{d^\top Qd}}
\le
2\sqrt{c^\top Q^{-1}c}.
\]
The result follows from \cref{lem:quadratic-identity}.
\end{proof}

In the body of the paper, \cref{cor:fp-obstruction} proves an upper bound for the normalized orthogonality deficit for floating-point product codes. This result used the following simple lemma, proved here for completeness, that consecutive floating-point values differ by at most a factor of two.
\begin{lemma}[Consecutive floating-point ratios]
\label{lem:fp-ratio}
Let $b\ge2$ and let $\Phi_{e,t}^+=\{a_1<\cdots<a_m\}$, where
$e\ge1$, $t\ge0$, and $e+t=b-1$. Then
\[
m=2^{b-1}-1
\]
and
\[
\frac{a_{j+1}}{a_j}\le2
\qquad(1\le j<m).
\]
\end{lemma}

\begin{proof}
The cardinality is
\[
(2^t-1)+2^t(2^e-1)=2^{e+t}-1=2^{b-1}-1.
\]
If $t=0$, the positive levels are powers of two and consecutive ratios are
exactly $2$. If $t\ge1$, consecutive denormal levels have ratios
$(k+1)/k\le2$; the transition from the last denormal level $2^t-1$ to the
first normal level $2^t$ has ratio $2^t/(2^t-1)<2$; consecutive levels within a
fixed exponent block have ratio at most $(2^t+m+1)/(2^t+m)<2$; and the jump
from the largest significand in one exponent block to the smallest significand
in the next has ratio
\[
\frac{2^{t+1}}{2^{t+1}-1}<2.
\]
Thus all consecutive ratios are at most $2$.
\end{proof}

Thus far, we have placed an upper bound on the cosine of the angular distance to a witness for floating-point alphabets. We will now prove a complementary lower bound for a carefully constructed family of alphabets consisting of successive powers of a base that grows very slowly with vector dimension. The proof follows a scale-integration argument~\cite{Devore1998} and a coordinate-wise maximization similar to that used in~\cite{min_codeword_alg} and extended in the final appendix of this paper. Eventually, by comparing these bounds, we will prove separation.

\begin{lemma}[$m$-level cosine lower bound]
\label{lem:block-hardy}
Fix $m\ge1$. Let $R_n\to\infty$ satisfy
\[
\log R_n=o(\log n),
\]
and define
\begin{align*}
B_{n,m}
&:= \{1,R_n,R_n^2,\dots,R_n^{m-1}\},\\
A_{n,m}
&:= \{0\}\cup \pm B_{n,m}.
\end{align*}
Then
\[
\alpha_n(A_{n,m})
\ge
\frac{2\sqrt m-o(1)}{\sqrt{H_n}}.
\]
Equivalently,
\[
\liminf_{n\to\infty}\sqrt{H_n}\cos F_n(A_{n,m})\ge2\sqrt m.
\]
\end{lemma}
\begin{proof}
Let $s=(s_1,\dots,s_n)\in[0,\infty)^n$ be arbitrary with $\|s\|_2=1$. For a
positive alphabet $B$, set
\[
M_B(s):=
\max_{x\in(\{0\}\cup B)^n\setminus\{0\}}
\frac{\langle s,x\rangle}{\|x\|_2}.
\]
We prove, uniformly over $s$, that
\[
M_{B_{n,m}}(s)\ge\frac{2\sqrt m-o(1)}{\sqrt{H_n}}.
\]
The claim for $A_{n,m}$ follows by applying this to $s_i=|u_i|$ and choosing
coordinate signs to agree with those of $u$.

Write $R=R_n$ and $B=\{1,R,\dots,R^{m-1}\}$. We will be working coordinate-wise. In order to do so, define
\[
\psi_B(y):=\max\left(0,\max_{b\in B}(2by-b^2)\right).
\]
When representing a particular coordinate $y$, we may interpret $\psi_B(y)$ as the improvement in squared-error we obtain by choosing a value from the positive alphabet $B$ rather than the value 0 to represent $y$, since $2by - b^2 = y^2 - (y-b)^2$.
Let $M=M_B(s)$. For every $x\in(\{0\}\cup B)^n$ and every scaling factor $\lambda>0$, completing the square gives,
\[
2\lambda\langle s,x\rangle-\|x\|_2^2
\le
2\lambda M\|x\|_2-\|x\|_2^2
\le
\lambda^2M^2.
\]
Maximizing over $x$ separates coordinatewise in the spirit of Gao~\cite{min_codeword_alg}, hence for $\lambda > 0$
\[
\sum_{i=1}^n\psi_B(\lambda s_i) = \max_{x\in (\{0\} \cup  B)^n}\left(2\lambda\langle s,x\rangle-\|x\|_2^2\right) \le\lambda^2M^2.
\]
Integrate this inequality against $d\lambda/\lambda^3$ over
$[\lambda_0,\lambda_1]$. After the change of variables $y=\lambda s_i$, we get
\[
\sum_{i=1}^n s_i^2
\int_{\lambda_0s_i}^{\lambda_1s_i}\psi_B(y)\frac{dy}{y^3}
\le
M^2\log\frac{\lambda_1}{\lambda_0}.
\]
Choose
\[
\lambda_0=\frac12,
\qquad
\lambda_1=\frac{R^m\sqrt n}{\eta_n},
\]
where $\eta_n\to0$ and $\log(1/\eta_n)=o(\log n)$; for instance,
$\eta_n=1/\log n$. Since $\psi_B(y)=0$ for $0\le y\le1/2$ and asymptotically every positive alphabet level is greater than 1, there is no advantage in selecting a nonzero alphabet for these coordinates. Let
\[
G:=\left\{i:s_i\ge\frac{\eta_n}{\sqrt n}\right\}.
\]
Then $\sum_{i\notin G}s_i^2\le\eta_n^2=o(1)$, while for $i\in G$ we have
\vspace{0.1pt}

\noindent$\lambda_1s_i\ge R^m$. Therefore
\[
\sum_{i=1}^n s_i^2
\int_{\lambda_0s_i}^{\lambda_1s_i}\psi_B(y)\frac{dy}{y^3}
\ge
(1-o(1))\int_0^{R^m}\psi_B(y)\frac{dy}{y^3}.
\]

It remains to estimate the last integral. For $j=0,\dots,m-2$, on the interval
$[R^j/2,R^{j+1}/2]$ we may use the level $b=R^j$, giving
\begin{align*}
\int_{R^j/2}^{R^{j+1}/2}(2R^j y - R^{2j})\frac{dy}{y^3}
&= \int_{1/2}^{R/2}(2t-1)\frac{dt}{t^3}\\
&= 2 - \frac4R + \frac2{R^2}.
\end{align*}
For the final level $R^{m-1}$, integrate over $[R^{m-1}/2,R^m]$ to obtain
\begin{align*}
\int_{R^{m-1}/2}^{R^m}(2R^{m-1}y-R^{2m-2})\frac{dy}{y^3}
&= \int_{1/2}^{R}(2t-1)\frac{dt}{t^3}\\
&= 2-\frac2R+\frac1{2R^2}.
\end{align*}
Consequently,
\[
\int_0^{R^m}\psi_B(y)\frac{dy}{y^3}
\ge
2m-o(1).
\]
Combining the preceding expressions gives
\[
M^2
\ge
\frac{2m-o(1)}{\log(\lambda_1/\lambda_0)}.
\]
Finally,
\begin{align*}
\log\frac{\lambda_1}{\lambda_0}
&= \log\left(\frac{2R^m\sqrt n}{\eta_n}\right)\\
&= \frac12\log n + o(\log n)\\
&= \frac12 H_n + o(H_n).
\end{align*}
because $m$ is fixed, $\log R_n=o(\log n)$, and
$\log(1/\eta_n)=o(\log n)$. Hence
\[
M^2\ge\frac{4m-o(1)}{H_n},
\]
and therefore
\[
M\ge\frac{2\sqrt m-o(1)}{\sqrt{H_n}}.
\]
The bound is uniform in $s$, so the lemma follows.
\end{proof}

In the body of the paper, \cref{cor:arb-lower} states the bound for the product code covering objective of an arbitrary alphabet. The proof for this result is given below, and uses the bound obtained in \cref{lem:block-hardy}.
\begin{proof}[\textbf{Proof of \cref{cor:arb-lower}} (page \pageref{cor:arb-lower})]
Let $m=2^{b-1}-1$. The alphabet $A_{n,m}$ from \cref{lem:block-hardy} has
$2m+1=2^b-1$ real values. Adjoin one additional scalar value not already in
$A_{n,m}$ to obtain an alphabet $\widetilde A_{n,m}$ with exactly $2^b$ values.
Since $A_{n,m}\subset\widetilde A_{n,m}$, enlarging the alphabet cannot
increase the covering radius, so
\[
\alpha_n(\widetilde A_{n,m})\ge\alpha_n(A_{n,m}).
\]
Taking the supremum over all $2^b$-element alphabets gives
\[
\cos w_{n,2^b}
\ge
\alpha_n(\widetilde A_{n,m})
\ge
\frac{2\sqrt m-o(1)}{\sqrt{H_n}}.
\]
This proves the claim.
\end{proof}

\section{Scaling to find minimum codeword}
\label{app:scale-min-codeword}
This section includes the proof of the method used for nearest codeword search in \cref{sec:experiments}, which is an extension of a result of~\cite{min_codeword_alg} to arbitrary alphabets with at least one positive and one negative element. The theorem below shows that the nearest codeword to a unit vector within the product direction set can be obtained by searching over all possible scaled versions of the elementwise quantized unit vector. 
\begin{theorem}[Scaling to find minimum codeword]
\label{thm:scale-min-codeword}
Let \(A\subset\mathbb R\) be finite with \(0\in A\). Assume that
\[
A\cap(0,\infty)\neq\varnothing,
\qquad
A\cap(-\infty,0)\neq\varnothing .
\]
Let \(Q:\mathbb R^n\to A^n\) denote elementwise nearest-neighbor
quantization onto \(A\). For \(u\in\mathbb S^{n-1}\), define
\[
\rho(u)
:=
\min_{x\in A^n\setminus\{0\}}
\angle(u,x).
\]
Then
\[
\rho(u)
=
\min_{\substack{s>0\\ Q(su)\neq 0}}
\angle\bigl(u,Q(su)\bigr).
\]
Equivalently, since \(s>0\) does not change direction,
\[
\rho(u)
=
\min_{\substack{s>0\\ Q(su)\neq 0}}
\angle\bigl(su,Q(su)\bigr).
\]
\end{theorem}

\begin{proof}
Since \(Q(su)\in A^n\), whenever \(Q(su)\neq0\) we have
\[
\rho(u)
\le
\angle\bigl(u,Q(su)\bigr).
\]
Taking the minimum over such \(s>0\) gives
\[
\rho(u)
\le
\min_{\substack{s>0\\ Q(su)\neq0}}
\angle\bigl(u,Q(su)\bigr).
\]
It remains to prove the reverse inequality. Let
\[
x^\star
\in
\argmin_{x\in A^n\setminus\{0\}}
\angle(u,x).
\]
Equivalently, \(x^\star\) maximizes
\[
\frac{\langle u,x\rangle}{\|x\|}
\]
over \(x\in A^n\setminus\{0\}\). Define
\[
c
:=
\frac{\langle u,x^\star\rangle}{\|x^\star\|}.
\]
We first show that \(c>0\). Since \(A\) contains both a positive
and a negative element, choose arbitrary
\[
a_+\in A\cap(0,\infty),
\qquad
a_-\in A\cap(-\infty,0).
\]
Define \(z\in A^n\) by
\[
z_i
:=
\begin{cases}
a_+, & u_i>0,\\
a_-, & u_i<0,\\
0,   & u_i=0.
\end{cases}
\]
Since \(u\neq0\), we have \(z\neq0\), and
\[
\langle u,z\rangle
=
\sum_{u_i>0} a_+u_i
+
\sum_{u_i<0} a_-u_i
>0.
\]
Hence
\[
\frac{\langle u,z\rangle}{\|z\|}>0.
\]
By optimality of \(x^\star\), it follows that \(c>0\).

Now set
\[
s^\star
:=
\frac{\|x^\star\|^2}
     {\langle u,x^\star\rangle}
=
\frac{\|x^\star\|}{c}.
\]
Since \(c>0\), we have \(s^\star>0\). We claim that \(x^\star\)
maximizes
\[
G(y)
:=
2s^\star\langle u,y\rangle-\|y\|^2
\]
over \(y\in A^n\).

Indeed, for every \(y\in A^n\setminus\{0\}\), angular optimality gives
\[
\frac{\langle u,y\rangle}{\|y\|}
\le c.
\]
Therefore
\[
\begin{aligned}
G(y)
&\le 2s^\star c\|y\|-\|y\|^2        \\
&= 2\|x^\star\|\|y\|-\|y\|^2        \\
&\le \|x^\star\|^2 .
\end{aligned}
\]
For \(y=0\), the same bound is immediate since \(G(0)=0\). On the
other hand,
\[
\begin{aligned}
G(x^\star)
&=
2s^\star\langle u,x^\star\rangle
-\|x^\star\|^2                         \\
&=
2\|x^\star\|^2-\|x^\star\|^2           \\
&=
\|x^\star\|^2 .
\end{aligned}
\]
Thus \(x^\star\) is a maximizer of \(G\) over \(A^n\).

Next observe that
\[
G(y)
=
\sum_{i=1}^n
\bigl(2s^\star u_i y_i-y_i^2\bigr).
\]
Hence maximizing \(G\) over \(A^n\) separates coordinatewise. For each
coordinate, maximizing
\[
a\mapsto 2s^\star u_i a-a^2
\]
over \(a\in A\) is equivalent to minimizing
\[
|a-s^\star u_i|^2
\]
over \(a\in A\), because
\[
2s^\star u_i a-a^2
=
(s^\star u_i)^2-|a-s^\star u_i|^2 .
\]
Therefore every nearest-neighbor quantization \(Q(s^\star u)\)
maximizes \(G\) over \(A^n\). 
Let
\[
q^\star:=Q(s^\star u).
\]
Then \(q^\star\) is a maximizer of \(G\), so
\[
G(q^\star)=\|x^\star\|^2.
\]
From the chain of inequalities above, equality can occur only if
\[
\|q^\star\|=\|x^\star\|,
\qquad
\frac{\langle u,q^\star\rangle}{\|q^\star\|}
=
c.
\]
In particular, \(q^\star\neq0\), and
\[
\angle(u,q^\star)
=
\angle(u,x^\star)
=
\rho(u).
\]
Hence
\[
\min_{\substack{s>0\\ Q(su)\neq0}}
\angle\bigl(u,Q(su)\bigr)
\le
\angle\bigl(u,Q(s^\star u)\bigr)
=
\rho(u).
\]
Combining both inequalities gives
\[
\rho(u)
=
\min_{\substack{s>0\\ Q(su)\neq0}}
\angle\bigl(u,Q(su)\bigr).
\]
Finally, since \(s>0\), the vectors \(u\) and \(su\) have the same
direction. Therefore
\[
\angle\bigl(su,Q(su)\bigr)
=
\angle\bigl(u,Q(su)\bigr),
\]
which proves the final equality.
\end{proof}

\end{document}